\documentclass[10pt,conference]{IEEEtran}
\def\BibTeX{{\rm B\kern-.05em{\sc i\kern-.025em b}\kern-.08em
    T\kern-.1667em\lower.7ex\hbox{E}\kern-.125emX}}

\usepackage{hyperref}

\usepackage{cite}
\usepackage{amsmath,amssymb,amsfonts}
\usepackage{xcolor}
\usepackage{color}
\usepackage{booktabs}
\usepackage{todonotes}
\usepackage{multirow} 
\usepackage{url}
\usepackage{textcomp}

\usepackage{balance}

\usepackage{longtable}
\usepackage{fancybox,framed}
\usepackage{tabu}
\usepackage{xspace}
\usepackage{bm}
\usepackage{bbm}      
\usepackage{amsthm}  

\newtheorem{theorem}{Theorem}

\theoremstyle{definition}\newtheorem{definition}{Definition}
\theoremstyle{remark}\newtheorem{example}{Example}
\newcommand{\ours}{\textsc{PatchPro}\xspace}

\usepackage{algorithm}  
\usepackage{algorithmicx}
\usepackage[noend]{algpseudocode}

\algrenewcommand{\algorithmiccomment}[1]{\hfill\(\triangleright\) \textcolor[gray]{0.5}{#1}}

\newcommand{\elmax}{\mathrm{elmax}}
\newcommand{\elmin}{\mathrm{elmin}}

\usepackage{algpseudocode}

\IEEEoverridecommandlockouts
\begin{document}
\pagestyle{plain}
\title{Patch Synthesis for Property Repair of Deep Neural Networks}
\DeclareRobustCommand*{\IEEEauthorrefmarkNum}[1]{%
  \raisebox{0pt}[0pt][0pt]{\textsuperscript{\footnotesize #1}}%
}

\author{
    \IEEEauthorblockN{Zhiming Chi\IEEEauthorrefmarkNum{1}\,\IEEEauthorrefmarkNum{2}, Jianan Ma\IEEEauthorrefmarkNum{3}\,\IEEEauthorrefmarkNum{4}, Pengfei Yang\IEEEauthorrefmarkNum{5}\IEEEauthorrefmark{1}\thanks{\IEEEauthorrefmark{1}is the corresponding author.}, Cheng-Chao Huang\IEEEauthorrefmarkNum{6}, Renjue Li\IEEEauthorrefmarkNum{1}\,\IEEEauthorrefmarkNum{2},\\ Jingyi Wang\IEEEauthorrefmarkNum{4}, Xiaowei Huang\IEEEauthorrefmarkNum{7} and Lijun Zhang\IEEEauthorrefmarkNum{1}}
    \IEEEauthorblockA{\IEEEauthorrefmarkNum{1}Key Laboratory of System Software (Chinese Academy of Sciences) and State Key\\
     Laboratory of Computer Science, Institute of Software, Chinese Academy of Sciences, Beijing, China
   }
    \IEEEauthorblockA{\IEEEauthorrefmarkNum{2}University of Chinese Academy of Sciences, Beijing, China
    }
    \IEEEauthorblockA{\IEEEauthorrefmarkNum{3}School of cyberspace, Hangzhou Dianzi University, Hangzhou, China 
    }
    \IEEEauthorblockA{\IEEEauthorrefmarkNum{4} Zhejiang University, Hangzhou, China
    }
    \IEEEauthorblockA{\IEEEauthorrefmarkNum{5}College of Computer and Information Science, Software College, Southwest University, Chongqing, China
    }
    \IEEEauthorblockA{\IEEEauthorrefmarkNum{6}Nanjing Institute of Software Technology, Chinese Academy of Sciences, Nanjing, China
    }
    \IEEEauthorblockA{\IEEEauthorrefmarkNum{7}University of Liverpool, Liverpool, United Kingdom
    \\Email: chizm@ios.ac.cn, majianannn@gmail.com, ypfbest001@swu.edu.cn, 
    chengchao@njis.ac.cn, \\lirj19@ios.ac.cn, wangjyee@zju.edu.cn, xiaowei.huang@liverpool.ac.uk, zhanglj@ios.ac.cn
    }
}

\maketitle
\begin{abstract}

Deep neural networks (DNNs) are prone to various dependability issues, such as adversarial attacks, which hinder their adoption in safety-critical domains. Recently, NN repair techniques have been proposed to address these issues while preserving original performance by locating and modifying guilty neurons and their parameters. However, existing repair approaches are often limited to specific data sets and do not provide theoretical guarantees for the effectiveness of the repairs.
To address these limitations, we introduce \ours, a novel patch-based approach for property-level repair of DNNs, focusing on local robustness. The key idea behind \ours is to construct patch modules that, when integrated with the original network, provide specialized repairs for all samples within the robustness neighborhood while maintaining the network's original performance. Our method incorporates formal verification and a heuristic mechanism for allocating patch modules, enabling it to defend against adversarial attacks and generalize to other inputs.
\ours demonstrates superior efficiency, scalability, and repair success rates compared to existing DNN repair methods, i.e., realizing provable property-level repair for 100\% cases across multiple high-dimensional datasets. 

\end{abstract}

\section{Introduction}

In recent years, deep neural networks (DNNs) have achieved significant advancements in various domains, including computer vision~\cite{badar2020application}, natural language processing~\cite{devlin2018bert}, and speech recognition~\cite{wang2017residual}. Despite these advancements, the adoption of DNNs in safety-critical domains has been slow due to concerns regarding their dependability. A major concern is the vulnerability of DNNs to adversarial attacks~\cite{pgd,autoattack}, where adversaries can manipulate input data in ways that are imperceptible to humans but can cause the model to make incorrect decisions. This vulnerability poses serious safety risks in applications such as autonomous vehicles~\cite{bojarski2016end} and medical diagnosis~\cite{vieira2017using}. 
Therefore, it is crucial to regularly update the DNN
to mitigate the risks associated with these errors and ensure the reliability of the network in practical applications.


DNN repair techniques~\cite{care,sohn2022arachne,vere_paper} have been proposed to address the problem
involving rectifying errors in the network by modifying its architecture or parameters. Compared to traditional methods such as adversarial training~\cite{gehr2018ai2,ganin2016domain,tramer2019adversarial}, input sanitization, fine-tuning, transfer learning~\cite{dai2009eigentransfer,ying2018transfer}, and data augmentation~\cite{yu2021deeprepair,ren2020few,ma2018mode},
neuron-level fault localization and repair methods~\cite{care} offer a more targeted approach by identifying and correcting errors at the individual neuron level while not affecting the overall network performance.

Despite significant advancements in this field, several limitations remain. 
First, neuron-level repair techniques, which rely on a limited number of samples, often struggle to provide robust defense against adversarial attacks due to the inherent complexity of these attacks compared to simpler threats like backdoor attacks. 
Adversarial attacks involve intricate mixtures of features~\cite{feature-purification}, making it difficult to generalize parameter adjustments from a small dataset. 
Second, existing repair methods typically focus on specific data for repair and fail to generalize to the property level (e.g., local robustness), which limits their effectiveness in addressing a broad range of adversarial scenarios. This motivates us to explore property repair of DNNs, specifically to fix the safety properties that neural networks violate in certain input regions. Third, while provable repair methods such as PRDNN~\cite{prdnn}, REASSURE~\cite{reassure}, and APRNN~\cite{aprnn} can conduct property-based error correction, they achieve this by ensuring that outputs meet constraints on the vertices of the input region as a polyhedron. However, the number of vertices increases exponentially with data dimensionality, which makes these methods inefficient for high-dimensional data.

In this work, we aim to 
address these limitations by proposing a novel patch-based method to achieve provable repair on the property level.
Specifically, we 
focus on correcting potential adversarial samples in the infinite set of high-dimensional points within a certain neighborhood of these error samples, i.e., satisfying local robustness property. To achieve this, our key idea is to use formal verification to help construct a separate patch module (in the form of a fully connected neural network structure) for each neighborhood outside of the original network.
Such a patch-based method allows us not only to ensure that the infinite set of points within the error sample's neighborhood repaired but also to maintain the performance of the original network unaffected. 
To construct such a patch, 
we utilize
reachability analysis through linear relaxation for the
provable training of these patch modules. Specifically, we use the
verification method DeepPoly~\cite{deeppoly} to create a linear relaxation of the output neurons. This approximation is employed to
determine the distance between the targeted behavior and the current
behavior, serving as the loss function. The patch modules are then
trained to minimize this distance. Once the loss function reaches zero, 
the patch modules will offer provable repairs for adversarial
attacks within the perturbation region.


To ensure that each patch module addresses specific neighborhoods, our approach employs an external indicator that identifies inputs within a particular sample’s local neighborhood and assigns the appropriate patch modules for repair. This indicator allows the same patch module to effectively repair adversarial attacks within the same perturbation region. For adversarial samples outside the repaired property, the indicator utilizes a heuristic allocation mechanism to assign suitable patch modules, thereby enhancing the generalization of the repair. This comprehensive framework not only improves the network's resilience against adversarial attacks but also maintains its accuracy. By leveraging reachability analysis and dedicated patch modules, we provide provable repairs for adversarial samples, significantly boosting the network's robustness. 
Additionally, to extend this local robustness across the entire dataset, we integrate the external indicator with the original network, ensuring that all samples, including those beyond the initially repaired property, benefit from heuristic patch module allocation. This approach effectively combines local and global repair strategies to enhance overall network robustness.

We summarize our contributions 
as follows:


\begin{itemize}
\item We introduce \ours, a novel approach that integrates a loss function derived from formal verification to train patch modules within the original network. This method effectively repairs high-dimensional infinite point sets within the polyhedron neighborhood of adversarial samples while preserving the network’s performance. The approach also includes a heuristic allocation mechanism, which ensures the generalization of local robustness repair by seamlessly integrating patch modules into the original network architecture.
\item To tackle the efficiency challenges of formal verification in large-scale DNNs, we utilize patch modules to perform repairs in the feature space of the networks. This strategy enables our method to scale effectively across various network architectures, ensuring both efficient and practical implementation.
\item We thoroughly evaluate \ours on three diverse datasets and multiple DNN architectures. Through extensive comparisons with state-of-the-art repair and adversarial training techniques, our method consistently demonstrates superior efficiency, scalability, and generalization capabilities. It shows significant improvement in handling general inputs, thereby greatly enhancing the overall robustness of the network. 
\end{itemize}

\section{Preliminary}
In this section we recall some basic notations of DNN repair. A deep neural network is a function ${N}\colon \mathbb{R}^{n_0} \rightarrow \mathbb{R}^{n_L}$ that maps an input $x \in \mathbb{R}^{n_0}$ to an output $y \in \mathbb{R}^{n_L}$. 
We usually visualize a DNN $N$ as a sequence of $L$ layers, where the $i$th layer contains $n_i$ neurons each representing a real variable.
Between two adjacent layers is typically a composition of an affine function and a non-linear activation function, and the DNN $N$ is the composition of the functions between layers.
In many applications, DNNs are serving for classification tasks. In such a classification DNN ${N}\colon \mathbb{R}^{n_0} \rightarrow \mathbb{R}^{n_L}$, every output dimension corresponds to a classification label, and the one with the maximum output value is the classification result that the DNN $N$ gives, i.e., $C_N(x)=\arg \max_{1 \le i \le n_L} N(x)_i$, where $N(x)_i$ is the $i$th entry of the vector $N(x)$.

The notion of safety properties pertains to assertions that 
guarantee the absence of undesirable behavior.  
Within the context of DNNs, a safety property demands 
that a DNN operates 
correctly within a specified input range. 
\begin{definition}
    A safety property is a triple $({N}, {X}, {Q})$, 
    where $N$ is a DNN, ${X} \subseteq \mathbb{R}^{d_{\text {in }}}$ 
    and ${Q} \subseteq \mathbb{R}^{d_{\text {out }}}$ 
    are the subset of input and output spaces of the neural network ${N}$.
    The property $({N}, {X}, {Q})$ is satisfied if and only if ${N}(x) \in {Q}$ for all $x \in {X}$.
\end{definition}

A local robustness property of a classification DNN $N$ requires that for any input $x$ in a given neighborhood 
$B(x_0,r)$ of an input $x_0$, its classification should always be 
consistent with $x_0$, where a neighborhood of an input $x_0$ is usually defined as 
a closed ball $B(x_0,r):=\{x \in \mathbb R^{n_0} \mid \|x-x_0\|_\infty \le r\}$, 
$\|\cdot\|_\infty$ is the $L_\infty$-norm, and $r>0$ is the 
radius. Formally, it can be defined as follows:

The local robustness of a DNN refers to its ability to maintain stable and consistent predictions in the vicinity of 
its in-distribution data points, even in the presence of small perturbations or variations in the input. 
Here a neighborhood of an input $x_0$ is usually defined as a closed ball $B(x_0,r):=\{x \in \mathbb R^{n_0} \mid \|x-x_0\|_\infty \le r\}$,  
where $\|\cdot\|_\infty$ is the $L_\infty$-norm, and $r>0$ is the radius.
Formally, a local robustness property of a classification DNN $N$ requires that for any input $x$ in a 
given neighborhood $B(x_0,r)$ of an input $x_0$, its classification should always be consistent with $x_0$, 
i.e., $\forall x \in B(x_0,r), C_N(x)=C_N(x_0)$. We denote this local robustness property as $(N,B(x_0,r))$, i.e.
$(N,B(x_0,r)) = (N, B(x_0,r), \{y \in \mathbb{R}^{n_L} \mid y_i < y_{C_N(x_0)} , i=1, 2, \dots, n_L\})$,
and thus it is a safety property.

In this work, we focus on the problem of repairing adversarial attacks with limited adversarial samples. Different from repairing backdoor attacks, not only does the buggy behavior of the given adversarial samples need fixing, but we also require that there should be no adversarial attacks around given adversarial samples, and that this enhancement of local robustness should generalize to samples across the whole dataset. Now we formally state the problem of repairing adversarial attacks as follows:

\begin{quote}
    Given a DNN $N$ and a set of adversarial samples $\{x_i^*\}_{i=1}^n$, where $n$ can be significantly smaller than the size of the training set, and each $x_i^*$ is obtained by an adversarial attack on an input $x_i$ with a given radius $r$, we need to construct a DNN $F$ which is locally robust on every $B(x_i,r)$ while the accuracy is maintained and local robustness of other inputs is potentially improved.
\end{quote}


\section{Methodology}

We focus on fixing adversarial attacks with limited data in this work, and the aims of the repair are at least threefold: For a given radius $r>0$,
\begin{itemize}
    \item the buggy behaviors of $\{x_i^*\}_{i=1}^n$ are fixed,
    \item the DNN is locally robust in $B(x_i,r)$ for $1 \le i \le n$,
    \item the accuracy of the DNN is maintained, and for as many input $x$ in the dataset as possible, the DNN is locally robust in $B(x,r)$.
\end{itemize}

In this work, we propose a patch-based repair method. A patch refers to a specific modification or alteration made to a software system or codebase; it is a discrete set of changes applied to fix a bug, enhance functionality, or address security vulnerabilities. Patch-based DNN repair involves the integration of an external indicator designed to identify buggy inputs, followed by the application of specialized patch modules to repair input sets which have similar behaviors. Each indicator here corresponds precisely to a robustness neighborhood $B(x_i,r)$ that requires repair, and we can leverage verification tools to ensure that the repair within each robustness neighborhood is provable. Furthermore, for other inputs in the dataset, we can heuristically match patches that maximize their robustness, thereby endowing this repair approach with global generalization. In this section, we propose a patch-based repair method named \ours to defend from adversarial attacks with limited data.


\subsection{Structure of the repaired DNN}

The main produce of the repaired network is shown in Fig.~\ref{fig:operation logic}. For an input $x \in \mathbb R^{n_0}$, the role of the indicator is to select the appropriate patches that, when applied, yields the sum of the outputs of the patch modules and the output of the original DNN for $x$. 
This sum represents the output obtained after the repair process. 
The indicator function is defined as:
\begin{figure}[h]
    \centering
    \includegraphics[width=0.98\linewidth]{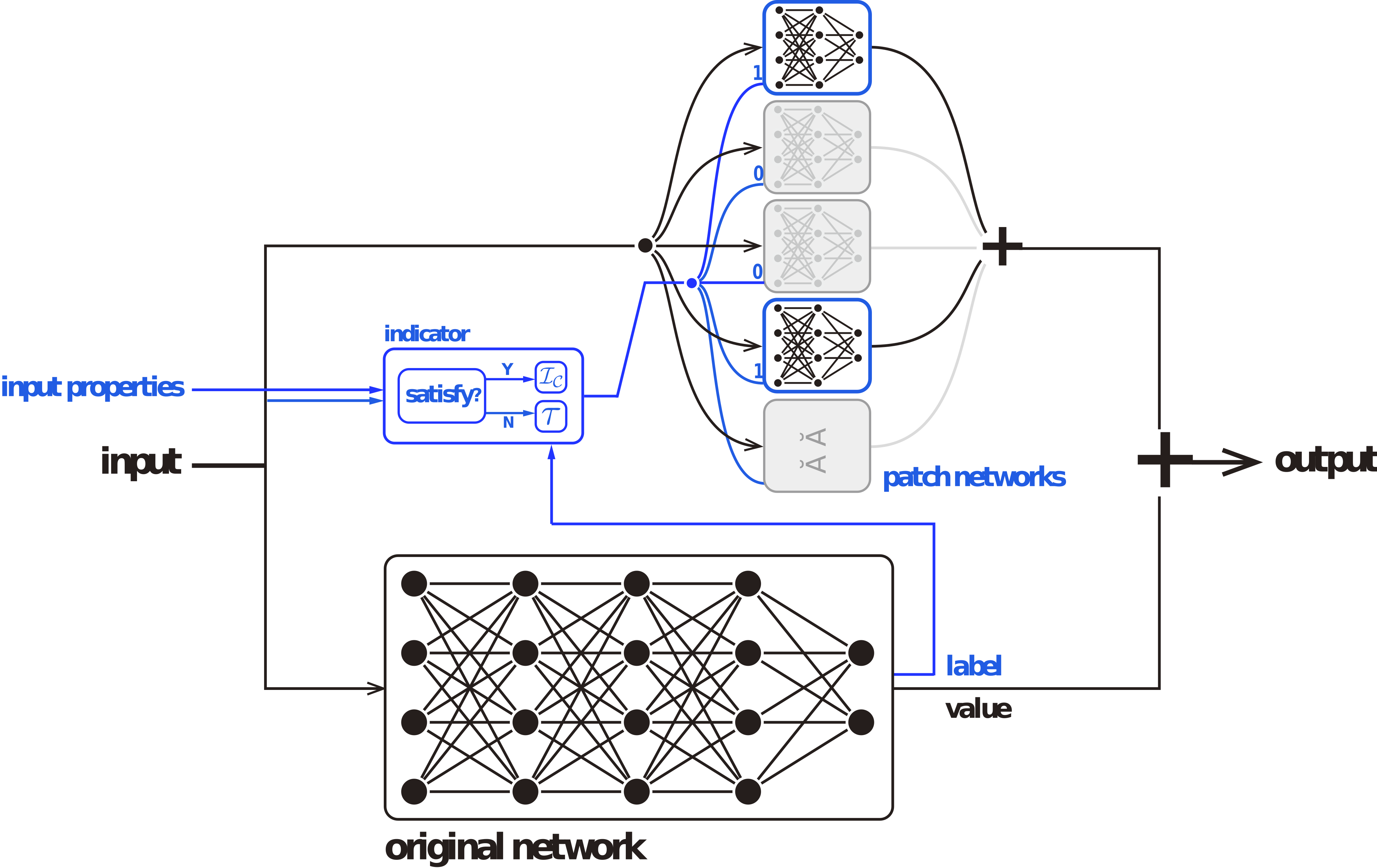}
    \caption{The architecture of a DNN repaired by \ours. 
    It contains multiple patch networks for each input properties. Each of them is enabled or disabled according to the allocation signal ``1'' or ``0'' determined by the indicator. The blue lines highlights the indicator's workflow.
    The final output is the sum of the outputs of all the enabled patches and the original network.
    }
    \label{fig:operation logic}
\end{figure}


\begin{definition} \label{def:Indicator}
    Let $\mathcal C=(X_1,\ldots,X_m)$ be a finite sequence of input properties. 
    The indicator function 
    $\mathcal{I}_{\mathcal{C}}\colon \mathbb{R}^{n_0} \to \{0,1\}^{m}$ outputs in the $j$th entry,  where $j \in \{1,\dots ,m\}$,  as
    $$
    \mathcal{I}_{\mathcal{C}}(x)_j = \left\{\begin{array}{ll}
    1, & \text { if } x \models X_j, \\
    0, & \text { otherwise}.
    \end{array}\right.
    $$
Typically an input property $X_i$ is a subset of $\mathbb R^{n_0}$, and we define $x \models X_i$ iff $x \in X_i$.
\end{definition}

Upon classification by the indicator, a set of patch modules is deployed to perform the repair. Here each patch module is specifically tailored 
to address an input set with the same local robustness property. The implementation of a patch module is a fully connected neural network in this work. For an input $x$ that
does not satisfy any input properties, 
i.e., $x \notin \bigcup_{i=1}^n B(x_i,r)$, 
there is no existing specific patch module,
but it may still suffer from adversarial attacks within its neighborhood.
In this situation, we 
heuristically allocate some patch modules to this input to defend from adversarial attacks. Formally, the structure of the repaired DNN is as follows:

\begin{definition} \label{def of F}
  A repaired DNN is a tuple $F=(N,\mathcal C,\mathcal P,\tau)$, where
  \begin{itemize}
      \item $N \colon \mathbb R^{n_0} \to \mathbb R^{n_L}$ is the original DNN,
      \item $\mathcal C=(X_1,\ldots,X_m)$ be a finite sequence of input properties,
      \item $\mathcal P = (P_1, \ldots , P_{m})^\mathrm{T}$ is a finite sequence of patch modules, each of which is a fully connected neural network,
      \item and $\tau \colon \mathbb R^{n_0} \setminus \bigcup_{i=1}^n X_i \to 2^{\{1,\ldots,m\}}$ is a patch allocation function.
  \end{itemize}
The semantics of the repaired DNN $F$ is a function:
\begin{align*}
   F\colon\mathbb R^{n_0} &\to \mathbb R^{n_L}, \\
           x    &\mapsto 
           \left\{\begin{array}{ll}
    N(x) + \mathcal I_{\mathcal C}(x)^\mathrm{T} \mathcal P(x), & \text { if } x \in \bigcup_{i=1}^n X_i, \\
    N(x) + \sum_{j \in \tau(x)}P_j(x), & \text { otherwise},
    \end{array}\right.
\end{align*}
where  $\mathcal P(x) = (P_1(x), \ldots , P_{m}(x))^\mathrm{T}$.
\end{definition}
In this work, we set $\mathcal C = \{B(x_i,r) \mid i=1,\ldots,n\}$, i.e., the indicator judges which robustness region the input belongs to, and the corresponding patch module defends against adversarial attacks in this specific region.

\subsection{Training the patch modules}

The first challenge is to train the patch modules, which are each implemented as a fully connected neural network, to repair each robustness property $(F,B(x_i,r))$ with a provable guarantee. To achieve provability, we employ formal verification in training these patch modules. The result of verification on a robustness property can be transformed into a loss function; once this loss reaches zero, the property is verified to be true, and the patch module in this stage can defend against all possible adversarial attacks in this robustness region. This idea has been widely adopted in DNN repair~\cite{dl2,art}. 

In this work, we employ DeepPoly~\cite{deeppoly} as the verification engine\footnote{Also there are several other verification tools based on abstract interpretation, 
such as CROWN~\cite{zhang2018efficient} and DeepZ~\cite{deepz}, which vary in precision, affecting loss function evaluations. 
We choose DeepPoly since it provides a good balance between efficiency and precision.}. 
DeepPoly uses abstract interpretation to give every neuron an upper/lower bound in the form of an affine function, where only variables in previous layers occur, and the numerical bound can be derived by propagating backwards these affine upper/lower bounds to the input layer. 
Recall that a local robustness property $(F,B(x_i,r))$ holds iff for any $x \in B(x_i,r)$ and $\ell \ne \ell_0 =C_F(x_i)$, $F(x)_\ell<F(x)_{\ell_0}$, so DeepPoly calculates the abstraction of $F$ on $B(x_i,r)$ for every neuron and the expressions $F(x)_\ell-F(x)_{\ell_0}$ for $\ell \ne \ell_0$. From the abstraction, we can obtain an affine function in which there are only input variables as a sound upper bound of $F(x)_\ell-F(x)_{\ell_0}$ on $B(x_i,r)$, i.e.,
\begin{align} \label{eq:deeppoly}
    \forall x \in B(x_i,r), F(x)_\ell-F(x)_{\ell_0} \le \alpha_\ell^\mathrm{T} x + \beta_\ell,
\end{align}
where $\alpha_\ell \in \mathbb R^{n_0}$ and $\beta_\ell \in \mathbb R$ are constants. It is easy to obtain the numerical upper bound of $\alpha_\ell^\mathrm{T} x + \beta_\ell$ on the box region $B(x_i,r)$. If this upper bound is negative for every $\ell \ne \ell_0$, then the local robustness property $(F,B(x_i,r))$ is verified to be true by DeepPoly. Therefore, it is natural to use this upper bound as the loss function to train the corresponding patch module. 


\begin{definition}  \label{def:loss}
    For a local robustness property $\varphi=(F,B(x_i,r))$, we define the safety violated loss function as
    \begin{align*}
        \mathcal L(\varphi)(x) = \sum_{\ell \ne \ell_0} \max (\alpha_\ell^\mathrm{T} x + \beta_\ell,0),
    \end{align*}
    where $\alpha_\ell^\mathrm{T} x + \beta_\ell$ is the upper bound of $F(x)_\ell-F(x)_{\ell_0}$ on $B(x_i,r)$ given by DeepPoly, as shown in Eq.~\eqref{eq:deeppoly}. For local robustness properties $\varphi_1,\ldots,\varphi_k$ which share the same ground truth label $\ell_0$, we define  $     \mathcal L(\bigwedge_{i=1}^k\varphi_i) = \sum_{i=1}^k \mathcal L(\varphi_i) $.
\end{definition}

For a local robustness property $\varphi=(F,B(x_i,r))$, its safety violated loss function $\mathcal L(\varphi)$ being $0$ implies that $\varphi$ is verified to be true by DeepPoly.
\begin{theorem}\label{thm:main}
    Let $\varphi=(F,B(x_i,r))$ be a local robustness property.
    If~$\mathcal L(\varphi) = 0$ on $B(x_i,r)$, i.e.,
    \begin{flalign*}
        \mathcal L^*(\varphi):=&\max (\elmax(\alpha_{\ell}^\mathrm{T},\bm 0) \cdot (x_i+r\cdot\bm 1)+ \elmin(\alpha_{\ell}^\mathrm{T},\bm 0) \cdot&\\
        &  (x_i-r\cdot \bm 1) + \beta_\ell,0) = 0,&        
    \end{flalign*}
    where $\elmax$ and $\elmin$ are the element-wise $\max$ and $\min$ operation, $\bm 0$ and $\bm 1$ are the vector in $\mathbb R^{n_0}$ with all the entries $0$ and $1$, respectively, then the property $\varphi$ holds. 
\end{theorem}

The notion that  $\mathcal L^*(\varphi)$  is the expansion of $\mathcal L(\varphi)(x) $ after taking corresponding values on the boundary of the ball. To improve the precision of $\mathcal L^*(\varphi)$ obtained from DeepPoly, we employ the input interval partitioning technique from ART~\cite{art}. It selects the partition dimension by computing the multiplication of the partial derivative of the safety violated loss function $\mathcal L(\varphi)(x)$ and the size of the input interval in the corresponding dimension, and bisects the box region over the dimension with the maximum score. After partitioning, the property \(\varphi\) is split into two new properties, whose input sets are two sub-boxes of the original input region, 
and the union of these two sub-boxes is the input set of \(\varphi\).

\begin{algorithm}[t]
    \caption{\ours } 
    \label{alg:advrepair}
    \begin{algorithmic}[1]
    \footnotesize
    \Require Original DNN $N$, pairs of input and its adversarial examples $\{(x_i,x_i^*)\}_{i=1}^n$, 
    and radius $r>0$
    \Ensure A repaired DNN $F=(N,\mathcal C, \mathcal P, \tau)$
    \State $\mathcal C \gets (B(x_1,r),\ldots,B(x_n,r))$, $\mathrm{iter} \gets 0$
    \State $\mathrm{Init}(\mathcal{P}) $    \Comment{Initialize the patch modules $\mathcal{P}=(P_1,\ldots,P_n)$}
    \State $D[\cdot,\cdot] \gets \{\}$
    \Comment{A dictionary where $D[i]$ stores properties fixed with $P_i$}

    \For {$i \gets 1  ~\text{to}~ n$}
        \State $D[i] \gets \{(F,B(x_i,r))\}$
    \EndFor
    \State $E \gets \{1,\ldots,n\}$     \Comment{Record the properties not fixed yet}
    \While{$\mathrm{iter} < M$}         \Comment{$M$: Maximum number of iterations}
        \State $\mathrm{iter} \gets \mathrm{iter}+1$
        \For{$j \in E$}
            \State $(P_j, T , \mathrm{Repaired}) \gets \textsc{Train}(N, P_j, D[j])$   \Comment{Alg.~\ref{alg:training}}

            \If{$\mathrm{Repaired} = \mathbf{True}$}
                \State $E \gets E \setminus \{j\}$
                \Else \For{$\varphi \in T$}     \Comment{We write $\varphi=(F,X)$} 
                   \For{$d \gets 1 ~\text{to}~ n_0$}
        \State $\operatorname{score}_d \gets \partial_d \mathcal L(\varphi)(x) \cdot \sup_{x,x' \in X} |x_d-x_d'|$   
    \EndFor
    \State $d^* \gets\arg\max_d\operatorname{score}_d$ 
    \State $X_1, X_2 \gets \mathrm{Bisect}(X,d^*)$     \Comment{Bisection on the $d^*$th dimension}
    \State $\varphi_1 \gets (F,X_1)$, $\varphi_2 \gets (F,X_2)$,
         $D[j] \gets (D[j] \cup \{\varphi_1,\varphi_2\}) \setminus \{\varphi\}$ 
        \EndFor

            \EndIf
       
        \EndFor
        \If{$E=\{\}$}
            \State \Return $F$    \Comment{Repair with provable guarantee}
        \EndIf
    \EndWhile
    \State \Return $F$     \Comment{Repair without provable guarantee}
    \end{algorithmic}
\end{algorithm}

\begin{algorithm}[t]
    \caption{Patch training}
    \footnotesize
    \label{alg:training}
    \begin{algorithmic}[1]
    \Require Original DNN $N$, DNN $P$ as a patch module, and a finite set $U$ of local robustness properties
    \Ensure The optimized patch module $P$, the set $T \subseteq U$ of properties to be refined, and 
 whether the properties in $U$ have all been fixed with a provable guarantee
    \Function{Train}{$N,P,U$}
    \State $\mathrm{epoch} \gets 0$
        \While{$\mathrm{epoch} < R$ }    \Comment{$R$: maximum number of epochs}
            \State $\mathrm{epoch} \gets \mathrm{epoch} + 1$ 
            \State $w \gets w-\eta \cdot \nabla \mathcal L^*(\bigwedge U)(w)$  \Comment{$w$: the weights in $P$}
            \If{$\mathcal L^*(\bigwedge U)(w) = 0$}           
            \Comment{$\mathcal L^*(\bigwedge U):= \sum_{\varphi \in U}\mathcal L^*(\varphi)$}
                \State \Return $(P, \{\}, \mathbf{True})$
            \EndIf
        \EndWhile
        \State $T \gets \mathrm{Slice}_K(\mathrm{ArgSort}_{\varphi}\{\mathcal L^*(\varphi)(w) \mid \mathcal L^*(\varphi)(w)>0, \varphi \in U\})$
    \State \Return $(P,T, \mathbf{False})$
    \EndFunction
    \end{algorithmic}
\end{algorithm}

The main algorithm of \ours is shown in Alg.~\ref{alg:advrepair}. We construct a dictionary $D$, where $D[i]$, initialized as $\{(F,B(x_i,r))\}$, stores the properties that the $i$th patch module $P_i$ repairs (Line 3--5, Alg.~\ref{alg:advrepair}). The repair process has at most $M$ iterations. In each iteration, we train the patches which do not yet provide provable repair, i.e., those whose index is in $E$. The algorithm of training a specific patch module $P$ for repairing a set $U$ of properties is shown in Alg.~\ref{alg:training}, which outputs the optimized patch module $P$, the set $T$ of properties to be refined, and a boolean label recording whether the repair of the patch $P$ has a provable guarantee. A standard gradient descent procedure is run until a provable repair is achieved, i.e., $\mathcal L^*(\bigwedge U)(w)=0$, or the number of epochs reaches a threshold $R$ (Line 3--7, Alg.~\ref{alg:training}). After that, if it is still not provable, we sort $\mathcal L^*(\varphi)(w)$ for $\varphi \in U$ from the largest to the smallest, and extract the largest $K$ ones (whose value is strictly larger than $0$) as the properties to be refined (Line 8, Alg.~\ref{alg:training}). For a property $\varphi=(F,X)$ to be refined, we select an input dimension to bisect the input space $X$. For an input dimension $d$, we define
\[
\mathrm{score}_d = \partial_d \mathcal L(\varphi)(x) \cdot \sup_{x,x' \in X} |x_d-x_d'|,
\]
where $\partial_d \mathcal L(\varphi)(x)$ is the partial derivative of $\mathcal L(\varphi)(x)$ on the $d$th dimension. We choose the dimension with the largest score to bisect $X$, and the property $\varphi=(F,X)$ is refined to two properties $\varphi_1=(F,X_1)$ and $\varphi_2=(F,X_2)$, recorded in the dictionary $D$ (Line~14--19, Alg.~\ref{alg:advrepair}). At the end of each iteration, we check whether all the patches provide provable repair; if so, it terminates immediately and outputs the current repaired DNN $F$, and this repair is provable (Line~20--21, Alg.~\ref{alg:advrepair}). If $E$ is still non-empty after $M$ iterations, it outputs  $F$ without provable guarantee (Line~22, Alg.~\ref{alg:advrepair}). Although we are focusing on fixing adversarial attacks in this work, \ours also works for repair safety properties of DNNs by making minor changes in Alg.~\ref{alg:advrepair}.

In \ours, we follow a classical way of defining the loss function with the linear relaxation 
obtained from DeepPoly, and incorporate it into our patch-based repair framework. 
This combination has several advantages. First, every patch is responsible for repairing 
properties in a specific pattern, which makes it much easier to obtain a provable repair; 
even if the repair is not provable, the safety violated loss is significantly declining in 
the training process, and it is highly possible that the repaired DNN is locally robust. 
Different from the neuron-level repair, there is no modification on the original DNN $N$ in \ours, 
and due to the design of the indicator, the patch modules do not affect the behaviors of 
other input before $\tau$ allocates a patch to it. This mechanism avoids the drawdown of accuracy, 
which is quite severe in many neuron-level repair methods. Also, by allocating patches with $\tau$, 
we can achieve good generalization of our repair to other inputs, and it is the key to solving this 
essential challenge in fixing adversarial attacks. In the following, we will present how to 
construct the patch allocation function $\tau$ to improve generalization.

For a network \(N\) with \(L\) layers and a maximum of \(n\) neurons per layer, the time complexity of running Deeppoly for one neuron is \(O(n^2 \cdot L)\)~\cite{deeppoly}. 
The time complexity of running Algorithm~\ref{alg:training} once, which involves backpropagation 
for gradient computation and update weights of the patch networks in line 5 of this algorithm, is \(O(R \cdot n_\mathrm{patch} \cdot L_\mathrm{patch})\), where
\(n_\mathrm{patch}\) is the maximum number of neurons per layer in all patch networks, and \(L_\mathrm{patch}\) is the number of layers in the patch networks.
Furthermore, during the execution of \ours, difficult-to-repair properties are split into two easier-to-repair properties, which may increase the number of properties to repair up to a predefined upper limit, denoted as \(K\).
In summary, the overall time complexity of \ours is \(O(R \cdot M \cdot K \cdot (n^3 \cdot L^2 + n_\mathrm{patch}^3 \cdot L_\mathrm{patch}^2))\).

\subsection{Patch allocation}\label{allocation}

For generalization purposes, it is necessary to introduce a patch allocation function $\tau$, given that the current patch construction approach lacks inherent adaptability to in-distribution data points.
This patch allocation function $\tau$ is the key to achieving good generalization of the patch modules trained locally to global inputs.

To match the best patch modules for an input $x' \notin \bigcup_{i=1}^n B(x_i,r)$, we propose utilizing the prediction $\ell_0 = C_N(x')$ of the original network on $x'$ as a guiding principle, with the aim of selecting patches that share the same ground truth label as $\ell_0$. Accordingly, we formally define the set of patch modules $\tau(x')$ associated with input $x'$ as $\tau(x') = \{i \mid  C_N(x') = C_N(x_i)\}$.
This definition establishes that the set $\tau(x')$ encompasses all indices $i$ satisfying the condition $C_N(x') = C_N(x_i)$, i.e., those instances where the prediction $\ell_0$ made by the original network on $x'$ is identical to the prediction made on sample $x_i$ at index $i$. By determining the patch module set $\tau(x')$ in this manner, we ensure consistency between the selected modules and the input $x'$ in terms of their predictions by the original network. This, in turn, is expected to enhance the robustness repair effectiveness for the specific input $x'$.


Once a set of patch modules has been allocated to an input $x'$, these modules effectively repair the entire robustness region of that input. Namely, we establish a new local robustness property such that any subsequent inputs $x \in B(x',r)$ will also utilize the same set of patch modules as $x'$, with the system's response given by $F(x) = \sum_{j \in \tau(x')} P_j(x) + N(x)$. 
In this context, the defense mechanism is established prior to the occurrence of adversarial attacks, obviating the need to re-allocate patch modules for each individual input $x$ within the robustness region, which would otherwise require computing $\tau(x)$ separately. This proactive construction of the defense ensures a consistent and efficient protection strategy against potential adversarial threats across the entire neighborhood of $x'$, reinforcing the overall robustness.

We seemingly have problem when $x'$ is not correctly classified by $N$, because in this case we may allocate inappropriate patches. If $x'$ is an adversarial example, we can employ sampling-based methods like \cite{mutation} to detect it and recognize its correct classification with a high probability. Otherwise, the wrong classification of $x'$ may result from backdoor attacks, biased training data, overfitting, etc, and such situations are beyond the scope of fixing adversarial attacks in this work.


\subsection{Repair in a feature space}
Repairing large DNNs with high-dimensional inputs poses significant challenges due to the substantial
memory and computational costs associated with formal verification techniques. These costs can often become 
prohibitive, especially when dealing with deep networks where the precision of abstraction-based methods, 
such as DeepPoly, deteriorates significantly over multiple layers of propagation.
Consequently, the key to adopting \ours to large DNNs is to reduce the input dimensionality and the size of the neural network in formal verification.

In the popular architectures of convolutional neural networks, the convolutional layers play the role of extracting important features from data, followed by several fully connected layers for classification according to these extracted features, so we call such a fully connected layer a feature space. 
\begin{definition}
    Let $N \colon \mathbb R^{n_0} \to \mathbb R^{n_L}$ be a DNN where the function from the $i$th layer to the $(i+1)$th is $f_i$, i.e., $N=f_{L-1} \circ \cdots \circ f_0$. For $l\in [1,L]$, the feature space of the $l$th layer is $\mathbb R^{n_l}$, and the behavior of the input $x \in \mathbb R^{n_0}$ in this feature space is $N_{0,l}:=f_{l-1}\circ \cdots \circ f_0(x)$. The neural network starting with the feature space of the $l$th layer in $N$ is $N_{l,L}=f_{L-1} \circ \cdots \circ f_l$.
\end{definition}
The dimensionality of a feature space is usually far smaller than that of the input layer. If we start with a feature space in employing DeepPoly, 
it does not need to calculate the abstraction of a great deal of convolutional layers, thus getting its efficiency and precision enhanced.
Selecting the input feature space for repairing remains an open challenge, with limited theoretical research available. 
Our choice of the second-to-last fully connected layer is heuristic, based on the common understanding that earlier layers in a neural 
network are primarily responsible for feature extraction, while later layers handle decision-making. 
Repairing in the decision layer is more 
efficient and incurs lower abstraction costs. 

To conduct repair in a given feature space with \ours, we first give an approximation of the feature space on every $B(x_i,r)$. 
Here we do not require that this approximation should be a sound abstraction of the real semantics of the feature space on $B(x_i,r)$, 
because calculating such an over-approximation with high precision is quite time-consuming. Instead, we simply sample in $B(x_i,r)$ and obtain the 
tightest box region that contains the buggy behaviors of the samples in the feature space as the approximation. 
To identify samples in the feature space that are ``close'' to adversarial examples, 
we use the Projected Gradient Descent (PGD) method and Fast Gradient Sign Method (FGSM) to find samples in $B(x_i,r)$,
and obtain their corresponding feature space behaviors. Formally, the abstraction $(X_{(i)})_j$ can be defined as follows.

\begin{definition}
Let $ S $ be the set of samples obtained by applying PGD and FGSM to $N$ in $ B(x_i, r) $. 
For each sample $ s \in S $, we compute its feature space behavior $ N_{0,l}(s) $ at the $ l $-th layer. 
The abstraction of $S$ in the feature space 
is defined as:
$$
(X_{(i)})_j = [ \min_{s \in S} (N_{0,l}(s))_j, \max_{s \in S} (N_{0,l}(s))_j ] 
$$
where $(X_{(i)})_j$ is the range of $j$-th dimension of $X_{(i)}$, $(N_{0,l}(s))_j$ is the $j$-th element of the $N_{0,l}(s)$, and $j= 1, \ldots ,n_l$.
\end{definition}

These samples, being actual adversarial examples in real-world scenarios, provide direct guidance for feature space repair.
This approximation is not sound yet, but it does not mean that the inputs in $B(x_i,r)$ whose behaviors in the feature space are not within 
this approximation will not be repaired by the corresponding patch module $P_i$, because the indicator $\mathcal I_{\mathcal C}$ still remains the same. 
Since the approximation of the feature space is not sound, the repair does not have a provable guarantee. It is also worth mentioning that, even if two properties with different ground truth labels have their feature spaces overlapping, our repair still works. In this situation, although two contradicted properties are involved in the training process, but they must correspond to two different robustness regions $B(x_i,r)$, and thus different patch modules $P_i$. The correspondence of the robustness regions $B(x_i,r)$ and the patch $P_i$ is always preserved, so $P_i$ is always working for repairing on $B(x_i,r)$ with the correct classification label $\ell_0$.

After obtaining an approximation $B_S$ for each $B(x_i,r)$ in the feature space, we employ Alg.~\ref{alg:advrepair}, where the initialization of $D[i]$ in Line~5 is $(N_{l,L}+P_i,X_{(i)})$ instead. Since the repair in the feature space is not provable, the safety violation loss function $\mathcal L(\varphi)$ for $\varphi=(N_{l,L}+P_i,X_{(i)})$ in Def.~\ref{def:loss} is modified to be
\[
 \mathcal L(\varphi)(x) = \sum_{\ell \ne \ell_0} (\alpha_\ell^\mathrm{T} x + \beta_\ell),
\]
where $\alpha_\ell^\mathrm{T} x + \beta_\ell$ for $\ell \ne \ell_0$ is obtained by DeepPoly abstracting $N_{l,L}+P_i$ on $X_{(i)}$, and $\mathcal L^*(\varphi)$ is modified accordingly. Also, we do not run Line~6--7 in Alg.~\ref{alg:training} and skip Line~6, 11--12 and 20--21 in Alg.~\ref{alg:advrepair}, so that the loss $\mathcal L^*$ decreases as much as possible and we can achieve a better repair performance.

In performing feature layer repairs on large-scale networks, we persist in utilizing the $\tau$ function guided by the original network output to allocate corresponding repair modules. While the approximation $X_{(i)}$ is indeed effective in capturing commonality in buggy behaviors, establishing a direct association between the input neighborhood $B(x', r)$ of a newly encountered point $x'$ and its distribution within the feature layers of the extensive network proves challenging. Consequently, assigning repair modules based on the input neighborhood of $x'$ emerges as a more targeted and reliable approach.


Repair in a feature space is an effective way to make \ours scale on large DNNs. Although we sacrifice provability, its repair performance and generalization capabilities demonstrate remarkable effectiveness in practical scenarios, which we will see in our experimental evaluation.

\begin{table*}[t]
\caption{Results of repairing local robustness.
We represent each tool's name with their first three letters, such as CAR for CARE, APR for APRNN, etc.
And ``--'' means timeout or memory overflow. } 
\label{tab:rq1-2}
\centering
    
\resizebox{1\linewidth}{!}
{
\begin{tabular}{c|c|r|rrrrrrr|rrrrrrr|rrrrrrr}
\hline
\multirow{2}{*}{Model}                                                & \multicolumn{1}{c|}{\multirow{2}{*}{$r$}} & \multicolumn{1}{c|}{\multirow{2}{*}{$n$}} & \multicolumn{7}{c|}{RSR/\%}                                                                                                                                                                        & \multicolumn{7}{c|}{DD/\%}                                                                                                                                                                         & \multicolumn{7}{c}{Time/s}                                                                                                                                                                        \\ \cline{4-24} 
                                                                      & \multicolumn{1}{l|}{}                        & \multicolumn{1}{l|}{}                     & \multicolumn{1}{c}{CAR} & \multicolumn{1}{c}{PRD} & \multicolumn{1}{c}{APR} & \multicolumn{1}{c}{REA} & \multicolumn{1}{c}{TRA} & \multicolumn{1}{c}{ART} & \multicolumn{1}{c|}{Ours} & \multicolumn{1}{c}{CAR} & \multicolumn{1}{c}{PRD} & \multicolumn{1}{c}{APR} & \multicolumn{1}{c}{REA} & \multicolumn{1}{c}{TRA} & \multicolumn{1}{c}{ART} & \multicolumn{1}{c|}{Ours} & \multicolumn{1}{c}{CAR} & \multicolumn{1}{c}{PRD} & \multicolumn{1}{c}{APR} & \multicolumn{1}{c}{REA} & \multicolumn{1}{c}{TRA} & \multicolumn{1}{c}{ART} & \multicolumn{1}{c}{Ours} \\ \hline
\multirow{15}{*}{\begin{tabular}[c]{@{}l@{}}FNN\_\\small\end{tabular}}& \multirow{5}{*}{0.05}                        & 50                                        & 8.0                      & \textbf{100.0}            & \textbf{100.0}            & \textbf{100.0}               & 54.0~~         & 10.0~             & \textbf{100.0}            & 1.0                      & 30.9                      & 19.7                      & \textbf{0.0}                 & 3.3          & 86.8             & \textbf{0.0}              & 2.0                      & 35.9                      & 3.5                       & 381.3                        & 5.2~         & 257.7~            & 14.3                     \\
                                                                      &                                              & 100                                       & 2.0                      & \textbf{100.0}            & \textbf{100.0}            & \textbf{100.0}               & 2.0~~          & 10.0~             & \textbf{100.0}            & 0.1                      & 25.8                      & 29.6                      & \textbf{0.0}                 & 86.8         & 87.0             & \textbf{0.0}              & 2.3                      & 80.8                      & 9.6                       & 762.5                        & 13.4~         & 318.5~            & 24.8                     \\
                                                                      &                                              & 200                                       & 4.5                      & \textbf{100.0}            & 0.0                       & \textbf{100.0}               & 3.5~~          & 11.0~             & \textbf{100.0}            & 0.4                      & 39.5                      & 66.2                      & \textbf{0.0}                 & 86.8         & 86.5             & \textbf{0.0}              & 3.2                      & 268.2                     & 82.8                      & 1486.6                       & 24.6~         & 422.3~            & 60.6                     \\
                                                                      &                                              & 500                                       & 5.8                      & \textbf{100.0}            & 0.0                       & \textbf{100.0}               & 2.8~~          & 11.0~             & \textbf{100.0}            & 1.7                      & 56.1                      & 68.3                      & \textbf{0.0}                 & 86.8         & 86.8             & \textbf{0.0}              & 8.2                      & 1358.4                    & 73.0                      & 3812.3                       & 54.1~         & 688.7~            & 236.7                    \\
                                                                      &                                              & 1000                                      & 11.4                     & \textbf{100.0}            & 0.0                       & --                           & 3.9~~          & 11.0~             & \textbf{100.0}            & 16.6                     & 43.1                      & 69.2                      & --                           & 86.8         & 87.0             & \textbf{0.0}              & 17.3                     & 3473.5                    & 34.3                      & --                           & 103.0~        & 1247.5~           & 756.7                    \\
                                                                      & \multirow{5}{*}{0.1}                         & 50                                        & 2.0                      & \textbf{100.0}            & \textbf{100.0}            & \textbf{100.0}               & 8.0~~          & 14.0~             & \textbf{100.0}            & \textbf{0.0}             & 71.3                      & 50.9                      & \textbf{0.0}                 & 86.8         & 86.5             & \textbf{0.0}              & 2.2                      & 51.6                      & 3.8                       & 358.8                        & 8.2~         & 393.0~            & 15.5                     \\
                                                                      &                                              & 100                                       & 2.0                      & \textbf{100.0}            & 0.0                       & \textbf{100.0}               & 13.0~~         & 9.0~              & \textbf{100.0}            & \textbf{0.0}             & 62.4                      & 48.5                      & \textbf{0.0}                 & 86.8         & 87.0             & \textbf{0.0}              & 1.7                      & 92.2                      & 72.8                      & 737.7                        & 13.4~         & 321.4~            & 37.3                     \\
                                                                      &                                              & 200                                       & 1.0                      & \textbf{100.0}            & 0.0                       & \textbf{100.0}               & 10.5~~         & 12.0~             & \textbf{100.0}            & 0.3                      & 60.1                      & 63.9                      & \textbf{0.0}                 & 86.8         & 86.5             & \textbf{0.0}              & 2.8                      & 321.8                     & 149.2                     & 1403.8                       & 22.7~         & 416.7~            & 62.5                     \\
                                                                      &                                              & 500                                       & 0.4                      & \textbf{100.0}            & 0.0                       & \textbf{100.0}               & 10.0~~         & 14.0~             & \textbf{100.0}            & \textbf{-0.1}            & 51.6                      & 71.8                      & 0.0                          & 86.8         & 86.8             & 0.0                       & 5.1                      & 1324.6                    & 27.9                      & 3482.9                       & 53.7~         & 564.4~            & 193.2                    \\
                                                                      &                                              & 1000                                      & 0.6                      & \textbf{100.0}            & 0.0                       & --                           & 10.0~~         & 14.0~             & \textbf{100.0}            & 0.1                      & 42.4                      & 73.0                      & --                           & 86.8         & 86.8             & \textbf{0.0}              & 8.0                      & 3844.8                    & 33.9                      & --                           & 104.5~        & 1733.2~           & 768.1                    \\
                                                                      & \multirow{5}{*}{0.3}                         & 50                                        & 0.0                      & \textbf{100.0}            & \textbf{100.0}            & \textbf{100.0}               & 8.0~~          & 8.0~              & \textbf{100.0}            & \textbf{0.0}             & 77.0                      & 42.9                      & \textbf{0.0}                 & 86.8         & 86.8             & \textbf{0.0}              & 1.8                      & 55.3                      & 3.6                       & 328.3                        & 13.3~         & 255.7~            & 18.6                     \\
                                                                      &                                              & 100                                       & 6.0                      & \textbf{100.0}            & \textbf{100.0}            & \textbf{100.0}               & 13.0~~         & 11.0~             & \textbf{100.0}            & 7.0                      & 74.1                      & 61.5                      & \textbf{0.0}                 & 86.8         & 87.0             & \textbf{0.0}              & 2.6                      & 115.7                     & 9.5                       & 670.3                        & 13.4~         & 322.3~            & 52.0                     \\
                                                                      &                                              & 200                                       & 0.0                      & \textbf{100.0}            & 0.0                       & \textbf{100.0}               & 10.5~~         & 8.0~              & \textbf{100.0}            & \textbf{0.0}             & 62.8                      & 54.2                      & \textbf{0.0}                 & 86.8         & 87.0             & \textbf{0.0}              & 3.2                      & 364.9                     & 132.1                     & 1306.2                       & 23.2~         & 405.6~            & 115.8                    \\
                                                                      &                                              & 500                                       & 0.0                      & \textbf{100.0}            & 0.0                       & \textbf{100.0}               & 10.0~~         & 10.0~             & \textbf{100.0}            & \textbf{-0.1}            & 57.1                      & 65.7                      & 0.0                          & 86.8         & 86.3             & 0.0                       & 4.8                      & 1934.4                    & 147.9                     & 3248.1                       & 52.5~         & 555.8~            & 255.9                    \\
                                                                      &                                              & 1000                                      & 0.0                      & \textbf{100.0}            & 0.0                       & --                           & 9.8~~          & 11.0~             & \textbf{100.0}            & \textbf{0.0}             & 58.3                      & 83.9                      & --                           & 86.8         & 86.8             & \textbf{0.0}              & 7.4                      & 3759.6                    & 44.7                      & --                           & 103.2~        & 2116.1~           & 935.8                    \\ \hline
\multirow{15}{*}{\begin{tabular}[c]{@{}l@{}}FNN\_\\ big\end{tabular}} & \multirow{5}{*}{0.05}                        & 50                                        & 6.0                      & \textbf{100.0}            & \textbf{100.0}            & \textbf{100.0}               & 98.0~~         & 24.0~             & \textbf{100.0}            & 2.0                      & 0.7                       & 3.6                       & \textbf{0.0}                 & 0.4          & 85.8             & \textbf{0.0}              & 1.7                      & 89.5                      & 11.2                      & 1807.6                       & 5.7~         & 436.5~            & 14.5                     \\
                                                                      &                                              & 100                                       & 10.0                     & \textbf{100.0}            & \textbf{100.0}            & \textbf{100.0}               & 78.0~~         & 9.0~              & \textbf{100.0}            & 2.8                      & 0.8                       & 16.0                      & \textbf{0.0}                 & \textbf{-0.1}& 86.9             & 0.0                       & 3.6                      & 171.1                     & 24.3                      & 3673.7                       & 14.3~         & 503.5~            & 29.1                     \\
                                                                      &                                              & 200                                       & 10.0                     & \textbf{100.0}            & \textbf{100.0}            & \textbf{100.0}               & 73.0~~         & 6.0~              & \textbf{100.0}            & 4.8                      & 1.0                       & 28.1                      & \textbf{0.0}                 & 0.1          & 88.3             & \textbf{0.0}              & 4.5                      & 419.0                     & 92.3                      & 7399.1                       & 24.3~         & 628.3~            & 51.2                     \\
                                                                      &                                              & 500                                       & 10.0                     & \textbf{100.0}            & \textbf{100.0}            & --                           & 67.0~~         & 10.0~             & \textbf{100.0}            & 6.1                      & 1.6                       & 25.8                      & --                           & \textbf{-0.1}& 87.1             & 0.0                       & 9.2                      & 1390.1                    & 657.6                     & --                           & 54.4~         & 1080.2~           & 290.6                    \\
                                                                      &                                              & 1000                                      & 11.2                     & \textbf{100.0}            & \textbf{100.0}            & --                           & 71.9~~         & 20.0~             & \textbf{100.0}            & 6.2                      & 1.9                       & 28.0                      & --                           & 0.1          & 85.8             & \textbf{0.0}              & 17.9                     & 4139.0                    & 5499.2                    & --                           & 108.1~        & 2167.6~           & 1033.2                   \\
                                                                      & \multirow{5}{*}{0.1}                         & 50                                        & 18.0                     & \textbf{100.0}            & \textbf{100.0}            & \textbf{100.0}               & 8.0~~          & 8.0~              & \textbf{100.0}            & 13.7                     & 6.7                       & 13.9                      & \textbf{0.0}                 & 0.8          & 87.4             & \textbf{0.0}              & 2.7                      & 88.2                      & 10.5                      & 1612.1                       & 7.1~         & 442.2~            & 13.9                     \\
                                                                      &                                              & 100                                       & 2.0                      & \textbf{100.0}            & \textbf{100.0}            & \textbf{100.0}               & 27.0~~         & 11.0~             & \textbf{100.0}            & 0.1                      & 5.6                       & 25.8                      & \textbf{0.0}                 & 8.8          & 87.4             & \textbf{0.0}              & 2.9                      & 173.7                     & 22.8                      & 3337.2                       & 13.5~         & 514.6~            & 32.1                     \\
                                                                      &                                              & 200                                       & 17.5                     & \textbf{100.0}            & \textbf{100.0}            & \textbf{100.0}               & 15.5~~         & 10.0~             & \textbf{100.0}            & 10.6                     & 6.1                       & 27.7                      & \textbf{0.0}                 & 0.7          & 87.4             & \textbf{0.0}              & 6.4                      & 445.5                     & 95.8                      & 6801.2                       & 23.3~         & 639.2~            & 75.4                     \\
                                                                      &                                              & 500                                       & 0.6                      & \textbf{100.0}            & \textbf{100.0}            & --                           & 15.4~~         & 15.0~             & \textbf{100.0}            & 1.3                      & 7.4                       & 33.7                      & --                           & 0.5          & 88.3             & \textbf{0.0}              & 6.3                      & 1403.6                    & 684.8                     & --                           & 57.1~         & 1048.6~           & 356.8                    \\
                                                                      &                                              & 1000                                      & 0.5                      & \textbf{100.0}            & \textbf{100.0}            & --                           & 26.4~~         & 11.0~             & \textbf{100.0}            & 0.7                      & 8.3                       & 33.3                      & --                           & 0.5          & 86.9             & \textbf{0.0}              & 11.2                     & 4321.6                    & 4726.4                    & --                           & 107.8~        & 2148.3~           & 735.9                    \\
                                                                      & \multirow{5}{*}{0.3}                         & 50                                        & 0.0                      & \textbf{100.0}            & \textbf{100.0}            & \textbf{100.0}               & 0.0~~          & 14.0~             & \textbf{100.0}            & \textbf{0.0}             & 27.5                      & 33.3                      & \textbf{0.0}                 & 0.5          & 85.8             & \textbf{0.0}              & 1.8                      & 119.5                     & 11.6                      & 1280.9                       & 7.1~         & 439.2~            & 17.1                     \\
                                                                      &                                              & 100                                       & 1.0                      & \textbf{100.0}            & \textbf{100.0}            & \textbf{100.0}               & 4.0~~          & 14.0~             & \textbf{100.0}            & \textbf{0.0}             & 31.8                      & 56.6                      & \textbf{0.0}                 & 1.8          & 85.8             & \textbf{0.0}              & 2.1                      & 253.1                     & 27.0                      & 2612.3                       & 13.5~         & 518.9~            & 37.0                     \\
                                                                      &                                              & 200                                       & 1.0                      & \textbf{100.0}            & \textbf{100.0}            & \textbf{100.0}               & 2.0~~          & 10.0~             & \textbf{100.0}            & 0.2                      & 28.5                      & 28.4                      & \textbf{0.0}                 & 0.4          & 86.9             & \textbf{0.0}              & 3.7                      & 630.5                     & 105.6                     & 5186.7                       & 23.9~         & 643.3~            & 100.4                    \\
                                                                      &                                              & 500                                       & 0.4                      & \textbf{100.0}            & \textbf{100.0}            & --                           & 2.2~~          & 9.0~              & \textbf{100.0}            & 0.2                      & 23.9                      & 29.5                      & --                           & 0.6          & 86.9             & \textbf{0.0}              & 4.9                      & 2117.7                    & 671.8                     & --                           & 55.2~         & 1047.6~           & 365.4                    \\
                                                                      &                                              & 1000                                      & 0.6                      & \textbf{100.0}            & \textbf{100.0}            & --                           & 4.8~~          & 12.0~             & \textbf{100.0}            & 0.7                      & 23.6                      & 26.3                      & --                           & 0.8          & 85.8             & \textbf{0.0}              & 10.9                     & 5480.3                    & 4297.3                    & --                           & 108.8~        & 2158.5~           & 1246.4                   \\ \hline
\multirow{15}{*}{\begin{tabular}[c]{@{}l@{}}CNN\end{tabular}}         & \multirow{5}{*}{0.05}                        & 50                                        & 0.0                      & \textbf{100.0}            & \textbf{100.0}            & \textbf{100.0}               & 100.0~~        & 0.0~              & \textbf{100.0}            & 2.4                      & 0.8                       & \textbf{0.0}              & \textbf{0.0}                 & 0.9          & 88.5             & \textbf{0.0}              & 3.5                      & 6.2                       & 57.9                      & 509.3                        & 4.3~         & 789.9~            & 14.1                     \\
                                                                      &                                              & 100                                       & 0.0                      & \textbf{100.0}            & \textbf{100.0}            & \textbf{100.0}               & 92.0~~         & 60.0~             & \textbf{100.0}            & 2.4                      & 0.7                       & 0.1                       & \textbf{0.0}                 & 0.2          & 88.5             & \textbf{0.0}              & 6.0                      & 5.1                       & 99.0                      & 910.7                        & 10.3~         & 879.0~            & 23.4                     \\
                                                                      &                                              & 200                                       & 0.0                      & \textbf{100.0}            & \textbf{100.0}            & \textbf{100.0}               & 88.0~~         & 0.0~              & \textbf{100.0}            & 2.7                      & 1.0                       & \textbf{0.0}              & \textbf{0.0}                 & 0.0          & 88.5             & \textbf{0.0}              & 9.8                      & 19.0                      & 182.9                     & 1825.6                       & 18.0~         & 969.9~            & 83.6                     \\
                                                                      &                                              & 500                                       & 0.0                      & \textbf{100.0}            & \textbf{100.0}            & --                           & 92.2~~         & --~               & \textbf{100.0}            & 3.4                      & 1.2                       & \textbf{-0.1}             & --                           & 0.0          & --               & 0.0                       & 26.9                     & 146.2                     & 547.8                     & --                           & 40.4~         & --~               & 221.9                    \\
                                                                      &                                              & 1000                                      & 0.0                      & \textbf{100.0}            & \textbf{100.0}            & --                           & 91.2~~         & --~               & \textbf{100.0}            & 7.5                      & 1.8                       & \textbf{0.0}              & --                           & 0.0          & --               & \textbf{0.0}              & 38.1                     & 549.1                     & 1611.6                    & --                           & 80.9~         & --~               & 1001.9                   \\
                                                                      & \multirow{5}{*}{0.1}                         & 50                                        & 0.0                      & \textbf{100.0}            & \textbf{100.0}            & \textbf{100.0}               & 94.0~~         & 0.0~              & \textbf{100.0}            & 0.6                      & 1.0                       & 0.1                       & \textbf{0.0}                 & 0.4          & 88.0             & \textbf{0.0}              & 2.6                      & 45.6                      & 56.6                      & 456.0                        & 5.7~         & 938.0~            & 12.7                     \\
                                                                      &                                              & 100                                       & 0.0                      & \textbf{100.0}            & \textbf{100.0}            & \textbf{100.0}               & 94.0~~         & 10.0~             & \textbf{100.0}            & 1.3                      & 1.0                       & 0.1                       & \textbf{0.0}                 & 0.7          & 88.7             & \textbf{0.0}              & 5.1                      & 9.8                       & 104.8                     & 896.8                        & 10.4~         & 1022.8~           & 29.5                     \\
                                                                      &                                              & 200                                       & 0.0                      & \textbf{100.0}            & \textbf{100.0}            & \textbf{100.0}               & 91.5~~         & 20.0~             & \textbf{100.0}            & 2.8                      & 2.2                       & 0.3                       & \textbf{0.0}                 & 0.4          & 88.5             & \textbf{0.0}              & 8.1                      & 41.1                      & 178.5                     & 1800.8                       & 17.7~         & 1108.1~           & 75.1                     \\
                                                                      &                                              & 500                                       & 0.0                      & \textbf{100.0}            & \textbf{100.0}            & --                           & 91.6~~         & --~               & \textbf{100.0}            & 9.1                      & 1.7                       & 0.4                       & --                           & 0.5          & --               & \textbf{0.0}              & 24.9                     & 139.2                     & 458.5                     & --                           & 40.6~         & --~               & 337.0                    \\
                                                                      &                                              & 1000                                      & 0.0                      & \textbf{100.0}            & \textbf{100.0}            & --                           & 92.8~~         & --~               & \textbf{100.0}            & 5.4                      & 2.4                       & 0.2                       & --                           & 0.5          & --               & \textbf{0.0}              & 32.0                     & 492.7                     & 1526.0                    & --                           & 80.4~         & --~               & 700.2                    \\
                                                                      & \multirow{5}{*}{0.3}                         & 50                                        & 0.0                      & \textbf{100.0}            & \textbf{100.0}            & \textbf{100.0}               & 0.0~~          & 10.0~             & \textbf{100.0}            & \textbf{0.0}             & 43.2                      & 7.2                       & \textbf{0.0}                 & 0.8          & 88.5             & \textbf{0.0}              & 2.2                      & 3.8                       & 59.9                      & 453.8                        & 5.6~         & 808.2~            & 15.0                     \\
                                                                      &                                              & 100                                       & 0.0                      & \textbf{100.0}            & \textbf{100.0}            & \textbf{100.0}               & 0.0~~          & 10.0~             & \textbf{100.0}            & \textbf{0.0}             & 38.8                      & 5.0                       & \textbf{0.0}                 & 0.2          & 89.3             & \textbf{0.0}              & 2.6                      & 29.4                      & 105.4                     & 907.5                        & 10.1~         & 899.5~            & 32.6                     \\
                                                                      &                                              & 200                                       & 0.0                      & \textbf{100.0}            & \textbf{100.0}            & \textbf{100.0}               & 0.0~~          & 0.0~              & \textbf{100.0}            & \textbf{0.0}             & 36.8                      & 3.6                       & \textbf{0.0}                 & 0.2          & 89.3             & \textbf{0.0}              & 4.6                      & 59.9                      & 196.2                     & 1793.7                       & 17.8~         & 1014.7~           & 87.9                     \\
                                                                      &                                              & 500                                       & 0.0                      & \textbf{100.0}            & \textbf{100.0}            & --                           & 0.2~~          & --~               & \textbf{100.0}            & 0.5                      & 48.8                      & 3.2                       & --                           & 0.3          & --               & \textbf{0.0}              & 9.2                      & 167.9                     & 518.8                     & --                           & 39.9~         & --~               & 361.4                    \\
                                                                      &                                              & 1000                                      & 0.0                      & \textbf{100.0}            & \textbf{100.0}            & --                           & 0.6~~          & --~               & \textbf{100.0}            & 0.1                      & 56.5                      & 2.2                       & --                           & 0.4          & --               & \textbf{0.0}              & 16.8                     & 423.8                     & 1667.0                    & --                           & 80.4~         & --~               & 1282.6                   \\ \hline
\end{tabular}
}

\vspace{-1.3em}
\end{table*}

\begin{table*}[t]
\caption{Results of correcting violation of safety properties on ACAS Xu}
\label{tab:acasxu}


\centering
    
\resizebox{1\linewidth}{!}
{ 
\begin{tabular}{c|rrrrr|rrrrr|rrrrr|rrrrr}
\hline
\multirow{2}{*}{Model} & \multicolumn{5}{c|}{RSR/\%}                                                                                                             & \multicolumn{5}{c|}{RGR/\%}                                                                                                             & \multicolumn{5}{c|}{FDD/\%}                                                                                                            & \multicolumn{5}{c}{Time/s}                                                                                                            \\ \cline{2-21} 
                       & \multicolumn{1}{c}{CAR} & \multicolumn{1}{c}{PRD} & \multicolumn{1}{c}{REA}  & \multicolumn{1}{c}{ART} & \multicolumn{1}{c|}{Ours} & \multicolumn{1}{c}{CAR} & \multicolumn{1}{c}{PRD} & \multicolumn{1}{c}{REA}  & \multicolumn{1}{c}{ART} & \multicolumn{1}{c|}{Ours} & \multicolumn{1}{c}{CAR} & \multicolumn{1}{c}{PRD} & \multicolumn{1}{c}{REA} & \multicolumn{1}{c}{ART} & \multicolumn{1}{c|}{Ours} & \multicolumn{1}{c}{CAR} & \multicolumn{1}{c}{PRD} & \multicolumn{1}{c}{REA} & \multicolumn{1}{c}{ART} & \multicolumn{1}{c}{Ours} \\ \hline
$N_{2,1}$              & 59.2                     & \textbf{100.0}     & \textbf{100.0} & \textbf{100.0}    & \textbf{100.0}            & 60.1                     & 97.5                & \textbf{100.0} & \textbf{100.0}    & \textbf{100.0}            & 39.1                     & 50.3                & 64.2          & 9.2               & \textbf{0.0}              & 28.1                     & 4.3     & 572.3         & 18.3              & 25.3                     \\
$N_{2,2}$              & 67.0                     & \textbf{100.0}     & \textbf{100.0} & \textbf{100.0}    & \textbf{100.0}            & 65.6                     & 99.0                & 99.7           & \textbf{100.0}    & \textbf{100.0}            & 56.5                     & 87.6                & 96.5          & 9.9               & \textbf{0.0}              & 23.6                     & 3.8     & 545.6         & 18.6              & 18.8                     \\
$N_{2,3}$              & \textbf{100.0}           & \textbf{100.0}     & \textbf{100.0} & \textbf{100.0}    & \textbf{100.0}            & 99.9                     & 99.6                & \textbf{100.0} & \textbf{100.0}    & \textbf{100.0}            & 11.1                     & 92.0                & 80.1          & 9.3               & \textbf{0.0}              & 14.2                     & 2.9     & 522.8         & 17.8              & 18.2                     \\
$N_{2,4}$              & \textbf{100.0}           & \textbf{100.0}     & \textbf{100.0} & \textbf{100.0}    & \textbf{100.0}            & 98.8                     & 99.9                & \textbf{100.0} & \textbf{100.0}    & \textbf{100.0}            & 10.5                     & 97.6                & 67.6          & 8.8               & \textbf{0.0}              & 13.9                     & 3.1     & 515.4         & 17.9              & 18.5                     \\
$N_{2,5}$              & 98.4                     & \textbf{100.0}     & \textbf{100.0} & \textbf{100.0}    & \textbf{100.0}            & 97.8                     & 98.9                & 64.1           & \textbf{100.0}    & \textbf{100.0}            & 9.6                      & 90.9                & 91.9          & 5.4               & \textbf{0.0}              & 28.8                     & 3.7     & 522.0         & 16.8              & 19.0                     \\
$N_{2,6}$              & 59.4                     & \textbf{100.0}     & \textbf{100.0} & \textbf{100.0}    & \textbf{100.0}            & 54.7                     & 98.5                & 87.3           & \textbf{100.0}    & \textbf{100.0}            & 5.1                      & 92.8                & 94.6          & 2.7               & \textbf{0.0}              & 29.2                     & 4.4     & 518.1         & 17.2              & 24.4                     \\
$N_{2,7}$              & \textbf{100.0}           & \textbf{100.0}     & \textbf{100.0} & \textbf{100.0}    & \textbf{100.0}            & \textbf{100.0}           & 99.3                & 92.5           & \textbf{100.0}    & \textbf{100.0}            & 5.9                      & 97.0                & 88.0          & 1.7               & \textbf{0.0}              & 11.7                     & 5.0     & 520.6         & 16.1              & 30.4                     \\
$N_{2,8}$              & \textbf{100.0}           & \textbf{100.0}     & \textbf{100.0} & \textbf{100.0}    & \textbf{100.0}            & \textbf{100.0}           & 97.8                & 93.7           & \textbf{100.0}    & \textbf{100.0}            & 6.3                      & 95.7                & 86.1          & 2.1               & \textbf{0.0}              & 12.7                     & 4.4     & 526.1         & 15.8              & 18.1                     \\
$N_{2,9}$              & 97.8                     & \textbf{100.0}     & \textbf{100.0} & \textbf{100.0}    & \textbf{100.0}            & 98.1                     & 90.5                & 92.6           & \textbf{100.0}    & \textbf{100.0}            & 13.5                     & 94.9                & 88.9          & 1.5               & \textbf{0.0}              & 24.4                     & 5.1     & 512.2         & 15.7              & 24.2                     \\
$N_{3,1}$              & \textbf{100.0}           & \textbf{100.0}     & \textbf{100.0} & \textbf{100.0}    & \textbf{100.0}            & \textbf{100.0}           & 97.8                & 99.2           & \textbf{100.0}    & \textbf{100.0}            & 8.7                      & 90.9                & 79.2          & 7.3               & \textbf{0.0}              & 12.7                     & 6.3     & 521.1         & 17.2              & 18.5                     \\
$N_{3,2}$              & \textbf{100.0}           & \textbf{100.0}     & \textbf{100.0} & \textbf{100.0}    & \textbf{100.0}            & \textbf{100.0}           & 99.3                & 95.3           & \textbf{100.0}    & \textbf{100.0}            & 8.9                      & 12.2                & 48.6          & 8.6               & \textbf{0.0}              & 12.2                     & 3.0     & 510.9         & 16.7              & 18.3                     \\
$N_{3,3}$              & \textbf{100.0}           & \textbf{100.0}     & \textbf{100.0} & \textbf{100.0}    & \textbf{100.0}            & 99.6                     & 94.3                & 98.8           & \textbf{100.0}    & \textbf{100.0}            & 35.7                     & 66.4                & 68.6          & 8.3               & \textbf{0.0}              & 24.6                     & 4.3     & 516.6         & 17.4              & 18.5                     \\
$N_{3,4}$              & \textbf{100.0}           & \textbf{100.0}     & \textbf{100.0} & \textbf{100.0}    & \textbf{100.0}            & \textbf{100.0}           & 98.1                & 97.7           & \textbf{100.0}    & \textbf{100.0}            & 18.9                     & 97.5                & 96.0          & 9.6               & \textbf{0.0}              & 23.3                     & 3.1     & 513.9         & 15.6              & 23.5                     \\
$N_{3,5}$              & 99.6                     & \textbf{100.0}     & \textbf{100.0} & \textbf{100.0}    & \textbf{100.0}            & 99.4                     & 97.7                & 80.7           & \textbf{100.0}    & \textbf{100.0}            & 13.0                     & 98.8                & 95.4          & 5.8               & \textbf{0.0}              & 33.8                     & 3.6     & 528.0         & 17.5              & 17.8                     \\
$N_{3,6}$              & \textbf{100.0}           & \textbf{100.0}     & \textbf{100.0} & \textbf{100.0}    & \textbf{100.0}            & 99.9                     & 98.7                & 97.2           & \textbf{100.0}    & \textbf{100.0}            & 9.8                      & 94.3                & 96.6          & 2.7               & \textbf{0.0}              & 16.1                     & 4.6     & 639.5         & 16.8              & 17.5                     \\
$N_{3,7}$              & \textbf{100.0}           & \textbf{100.0}     & \textbf{100.0} & \textbf{100.0}    & \textbf{100.0}            & \textbf{100.0}           & 91.6                & 81.8           & \textbf{100.0}    & \textbf{100.0}            & 5.5                      & 96.4                & 80.0          & 2.2               & \textbf{0.0}              & 20.6                     & 5.0     & 515.4         & 16.9              & 18.6                     \\
$N_{3,8}$              & \textbf{100.0}           & \textbf{100.0}     & \textbf{100.0} & \textbf{100.0}    & \textbf{100.0}            & \textbf{100.0}           & 98.8                & 98.3           & \textbf{100.0}    & \textbf{100.0}            & 5.0                      & 97.5                & 88.8          & 1.6               & \textbf{0.0}              & 15.2                     & 4.1     & 579.9         & 17.9              & 24.8                     \\
$N_{3,9}$              & \textbf{100.0}           & \textbf{100.0}     & \textbf{100.0} & \textbf{100.0}    & \textbf{100.0}            & 99.6                     & 96.0                & 49.2           & \textbf{100.0}    & \textbf{100.0}            & 9.2                      & 97.0                & 86.1          & 1.9               & \textbf{0.0}              & 16.2                     & 4.9     & 774.4         & 16.8              & 19.0                     \\
$N_{4,1}$              & \textbf{100.0}           & \textbf{100.0}     & \textbf{100.0} & \textbf{100.0}    & \textbf{100.0}            & 99.9                     & 98.3                & 98.8           & \textbf{100.0}    & \textbf{100.0}            & 40.1                     & 96.2                & 93.2          & 7.8               & \textbf{0.0}              & 14.5                     & 3.5     & 549.4         & 17.6              & 24.5                     \\
$N_{4,3}$              & 98.4                     & \textbf{100.0}     & \textbf{100.0} & \textbf{100.0}    & \textbf{100.0}            & 97.5                     & 99.4                & 99.8           & \textbf{100.0}    & \textbf{100.0}            & 20.1                     & 91.1                & 87.5          & 7.0               & \textbf{0.0}              & 28.0                     & 2.9     & 506.7         & 17.1              & 17.7                     \\
$N_{4,4}$              & \textbf{100.0}           & \textbf{100.0}     & \textbf{100.0} & \textbf{100.0}    & \textbf{100.0}            & \textbf{100.0}           & 98.7                & 99.9           & \textbf{100.0}    & \textbf{100.0}            & 4.3                      & 82.0                & 85.2          & 8.4               & \textbf{0.0}              & 9.5                      & 3.1     & 535.8         & 16.9              & 17.9                     \\
$N_{4,5}$              & \textbf{100.0}           & \textbf{100.0}     & \textbf{100.0} & \textbf{100.0}    & \textbf{100.0}            & \textbf{100.0}           & 99.8                & 94.6           & \textbf{100.0}    & \textbf{100.0}            & 33.6                     & 46.4                & 94.8          & 6.0               & \textbf{0.0}              & 9.8                      & 5.0     & 544.8         & 18.2              & 17.7                     \\
$N_{4,6}$              & \textbf{100.0}           & \textbf{100.0}     & \textbf{100.0} & \textbf{100.0}    & \textbf{100.0}            & \textbf{100.0}           & 98.2                & 97.9           & \textbf{100.0}    & \textbf{100.0}            & 11.2                     & 89.6                & 86.0          & 2.6               & \textbf{0.0}              & 19.9                     & 5.4     & 548.2         & 15.9              & 24.1                     \\
$N_{4,7}$              & 77.6                     & \textbf{100.0}     & \textbf{100.0} & \textbf{100.0}    & \textbf{100.0}            & 78.1                     & 98.8                & 85.6           & \textbf{100.0}    & \textbf{100.0}            & 3.9                      & 96.3                & 96.2          & 1.6               & \textbf{0.0}              & 38.4                     & 4.5     & 544.0         & 16.0              & 24.5                     \\
$N_{4,8}$              & 99.8                     & \textbf{100.0}     & \textbf{100.0} & \textbf{100.0}    & \textbf{100.0}            & 99.7                     & 97.4                & 77.1           & \textbf{100.0}    & \textbf{100.0}            & 4.7                      & 97.9                & 88.2          & 1.6               & \textbf{0.0}              & 36.0                     & 4.9     & 542.7         & 14.8              & 30.2                     \\
$N_{4,9}$              & 99.8                     & \textbf{100.0}     & \textbf{100.0} & \textbf{100.0}    & \textbf{100.0}            & 99.8                     & 95.0                & 64.0           & \textbf{100.0}    & \textbf{100.0}            & 7.7                      & 98.4                & 90.1          & 3.3               & \textbf{0.0}              & 21.2                     & 5.0     & 545.8         & 16.4              & 23.8                     \\
$N_{5,1}$              & \textbf{100.0}           & \textbf{100.0}     & \textbf{100.0} & \textbf{100.0}    & \textbf{100.0}            & \textbf{100.0}           & 90.5                & 97.6           & \textbf{100.0}    & \textbf{100.0}            & 14.0                     & 98.4                & 96.9          & 6.9               & \textbf{0.0}              & 19.8                     & 3.5     & 541.3         & 17.1              & 18.6                     \\
$N_{5,2}$              & \textbf{100.0}           & \textbf{100.0}     & \textbf{100.0} & \textbf{100.0}    & \textbf{100.0}            & 99.9                     & 99.3                & \textbf{100.0} & \textbf{100.0}    & \textbf{100.0}            & 4.1                      & 91.8                & 94.1          & 7.2               & \textbf{0.0}              & 13.3                     & 4.0     & 543.4         & 17.2              & 18.3                     \\
$N_{5,3}$              & \textbf{100.0}           & \textbf{100.0}     & \textbf{100.0} & \textbf{100.0}    & \textbf{100.0}            & 99.9                     & 88.6                & 99.7           & \textbf{100.0}    & \textbf{100.0}            & 8.8                      & 88.3                & 94.8          & 6.0               & \textbf{0.0}              & 17.4                     & 3.5     & 542.9         & 16.8              & 18.7                     \\
$N_{5,4}$              & 99.6                     & \textbf{100.0}     & \textbf{100.0} & \textbf{100.0}    & \textbf{100.0}            & 98.9                     & 95.9                & 99.5           & \textbf{100.0}    & \textbf{100.0}            & 17.7                     & 98.0                & 91.9          & 7.1               & \textbf{0.0}              & 16.9                     & 2.9     & 540.7         & 16.1              & 18.2                     \\
$N_{5,5}$              & \textbf{100.0}           & \textbf{100.0}     & \textbf{100.0} & \textbf{100.0}    & \textbf{100.0}            & \textbf{100.0}           & 97.0                & 99.7           & \textbf{100.0}    & \textbf{100.0}            & 5.4                      & 94.8                & 94.4          & 3.9               & \textbf{0.0}              & 8.6                      & 3.4     & 518.9         & 15.9              & 18.3                     \\
$N_{5,6}$              & \textbf{100.0}           & \textbf{100.0}     & \textbf{100.0} & \textbf{100.0}    & \textbf{100.0}            & \textbf{100.0}           & 99.2                & 98.8           & \textbf{100.0}    & \textbf{100.0}            & 6.0                      & 95.7                & 93.7          & 9.6               & \textbf{0.0}              & 13.1                     & 4.7     & 519.5         & 15.9              & 30.0                     \\
$N_{5,7}$              & 97.2                     & \textbf{100.0}     & \textbf{100.0} & \textbf{100.0}    & \textbf{100.0}            & 97.7                     & 99.1                & 98.6           & \textbf{100.0}    & \textbf{100.0}            & 4.9                      & 95.3                & 87.7          & 1.8               & \textbf{0.0}              & 21.2                     & 5.0     & 517.3         & 16.3              & 24.2                     \\
$N_{5,8}$              & 99.2                     & \textbf{100.0}     & \textbf{100.0} & \textbf{100.0}    & \textbf{100.0}            & 98.9                     & 99.7                & 95.5           & \textbf{100.0}    & \textbf{100.0}            & 5.2                      & 99.2                & 80.0          & 1.8               & \textbf{0.0}              & 28.3                     & 4.3     & 517.9         & 15.9              & 24.1                     \\
$N_{5,9}$              & \textbf{100.0}           & \textbf{100.0}     & \textbf{100.0} & \textbf{100.0}    & \textbf{100.0}            & \textbf{100.0}           & 96.9                & 91.7           & \textbf{100.0}    & \textbf{100.0}            & 5.4                      & 97.2                & 27.0          & 2.0               & \textbf{0.0}              & 22.6                     & 4.8     & 515.8         & 15.4              & 24.3                     \\ \hline
Avg                    & 95.8                     & \textbf{100.0}     & \textbf{100.0} & \textbf{100.0}    & \textbf{100.0}            & 95.5                     & 97.3                & 92.2           & \textbf{100.0}    & \textbf{100.0}            & 13.4                     & 88.8                & 85.1          & 5.2               & \textbf{0.0}              & 20.0                     & 4.2     & 540.9         & 16.8              & 21.4                     \\ \hline
\end{tabular}
}
\vspace{-1em}

\end{table*}

\section{Experimental Evaluation}

In this section, we evaluate \ours by answering the following research questions:

\noindent \textbf{RQ1}: What is the overall performance of \ours in repairing local robustness and correcting safety property violations?

\noindent \textbf{RQ2}: Does the repaired DNN exhibit the capability to defend against new adversarial attacks?


\noindent \textbf{RQ3}: Does \ours have scalability to repair large networks?

\noindent \textbf{RQ4}: In the context of \ours, how does the size and quantity of patch modules influence the efficacy of the repair?


\subsection{Setup}
All the experiments are conducted on a machine with an AMD EPYC 7763 64-Core Processor, 256 GB of memory, and an NVIDIA GeForce RTX 3090 with 24 GB of GPU memory. Each experiment has a timeout set to 10\,000 seconds.

\paragraph{Dataset}
We conduct evaluations on four common datasets: 
MNIST~\cite{MNIST}, CIFAR-10~\cite{cifar10}, Tiny ImageNet\cite{tinyimagenet}, and ACAS Xu~\cite{acasxu14,acasxu15}. 
The first three are widely recognized benchmarks in the field for studying neural network robustness~\cite{adversarialtrain,pgd}, and 
ACAS Xu is a commonly used benchmark in research on neural network verification and repair~\cite{reluplex,aprnn,prdnn}.
On MNIST, we use two fully connected networks and one convolutional network, while on CIFAR-10, we train a VGG19~\cite{vgg} and a ResNet18~\cite{resnet}.
On Tiny ImageNet, we train a Resnet152 and a WRN101-2~\cite{WRN}.
We assess our approach on 35 ACAS Xu DNNs, which, as documented in \cite{reluplex}, are expected to satisfy Property-2 but exhibit violations.
For the local robustness repair task, we generate adversarial samples using PGD~\cite{pgd} for DNNs trained on MNIST and CIFAR-10 (50, 100, 200, 500, and 1,000 samples), as well as Tiny ImageNet (500 and 1,000 samples). The radius $r$ is set to 0.05, 0.1, and 0.3 for MNIST, $\frac{4}{255}$ and $\frac{8}{255}$ for CIFAR-10, and $\frac{2}{255}$ and $\frac{4}{255}$ for Tiny ImageNet.
On ACAS Xu, we aim to repair the violation of Property-2 using one patch module while preserving the original performance. 
Although \ours repairs safety properties without requiring sample information, due to the needs of other tools such as CARE, we sample 500 counterexamples as the faulty inputs for repair.
For CIFAR-10, the accuracies of VGG19 and ResNet18 are 93.4\% and 88.3\%, respectively. For Tiny ImageNet, the accuracies of WRN101-2 and ResNet152 are 64.4\% and 68.2\%, respectively.
The maximum number of iterations $M$ in Alg.~\ref{alg:advrepair} is set to 25. 
In Alg.~\ref{alg:training}, the maximum number of epochs $R$, learning rate $\eta$, and selection number $K$ are set to 10, 10, and 800, respectively. All patch modules consist of a single linear layer, unless otherwise specified. On MNIST, the patch module takes the sample itself as input, while on CIFAR-10 and Tiny ImageNet, it takes the output from the network's penultimate layer as input.
For sampling the feature space, we use PGD attack with a step size of $\frac{2}{255}$ for 10 rounds, 
with each round consisting of 50 steps. We collect all adversarial examples generated during this process. 
Additionally, we use FGSM attack to generate 50 adversarial examples, which 
are then combined with the samples obtained from PGD. This combined set of samples is used for sampling the feature space.

\paragraph{Baselines} We compare \ours with the state-of-the-art repair methods including
CARE~\cite{care}, APRNN~\cite{aprnn}, PRDNN~\cite{prdnn}, REASSURE~\cite{reassure} and ART~\cite{art}. 
Since APRNN and PRDNN require selecting a layer to repair on, 
we traverse all eligible layers across each baseline, 
and report the layer that exhibit the best performance. Additionally, for the local robustness repair task, we compare our approach with an adversarial training method 
TRADES~\cite{trades}, where we select the trained model with the best performance in 200 epochs with their default parameters.

\paragraph{Metrics} 
To assess generalization, we sample 10 adversarial examples on each $B(x_i,r)$ different from $x_i^*$ to form a generalization set $D_\mathrm{g}$ for MNIST, CIFAR-10 and Tiny ImageNet, and use an independent set of 5\,000  counterexamples for ACAS Xu.
Besides these, we have an independent test set $D_\mathrm{t}$ of size 10\,000, 10\,000, 10\,000 and 5\,000 for MNIST, CIFAR-10, Tiny ImageNet and ACAS Xu, respectively. We employ 
AutoAttack~\cite{autoattack} to attack the repaired DNN $F$ on $\{x_i\}_{i=1}^n$ and $D_\mathrm{t}$ with the same radius $r$.
The metrics we use include repair success rate ($\mathrm{RSR}$), repair generalization rate ($\mathrm{RGR}$), drawdown ($\mathrm{DD}$), defense success rate ($\mathrm{DSR}$), and defense generalization success rate ($\mathrm{DGSR}$), defined as:
%

\footnotesize
\vspace{-8pt}
    \begin{flalign*}
    \mathrm{RSR}&= \frac{ |\{\,i \mid C_{F}(x_i^*) = C_N(x_i)\,\}|}{n},& \\
    \mathrm{RGR}&= \frac{|\{{\,x \in D_{\mathrm{g}}} \mid C_{F}(x) = \ell_x\,\}|}{\left|D_{\mathrm{g}}\right|},& \\
\mathrm{DD} &= \frac{|\{\,{x \in D_{\mathrm{t}}} \mid C_{N}(x) = \ell_x\,\}| - |\{\,{x \in D_{\mathrm{t}}} \mid C_{F}(x) = \ell_x\,\}|}{\left|D_{\mathrm{t}}\right|}, &\\
\mathrm{DSR} &= \frac{|\{\,i \mid \forall x \in \mathrm{AA}(F,B(x_i,r)), C_{F}(x) = C_N(x_i)\,\}|}{n},& \\
 \mathrm{DGSR}&= \frac{|\{\,x \in D_{\mathrm{t}} \mid \forall x' \in \mathrm{AA}(F,B(x,r)), C_{F}(x') = C_N(x)\,\}|}{\left|D_{\mathrm{t}}\right|}, &
\end{flalign*}
\normalsize
%
%
%
where $\ell_x$ denotes the ground truth of $x$, and $\mathrm{AA}(\varphi)$ represents the set of potential adversarial samples obtained by attacking the local robustness property $\varphi$ with AutoAttack.
On ACAS Xu, the metric Drawdown is replaced with Fidelity Drawdown ($\mathrm{FDD}$) because there is no labelled ground truth for inputs on ACAS Xu:

\footnotesize
\[
 \mathrm{FDD} = \frac{|\{\,{x \in D_{\mathrm{t}}} \mid C_{F}(x) \neq C_{N}(x)\,\}|}{|D_{\mathrm{t}}|}.
\]
\normalsize
The metrics $\mathrm{RSR}$ and $\mathrm{DD}$ (or $\mathrm{FDD}$) evaluates the overall performance of DNN repair by measuring the percentage of buggy inputs successfully repaired and how much the accuracy decreases, while $\mathrm{RGR}$, $\mathrm{DSR}$ and $\mathrm{DGSR}$ reflects how the repair generalizes to the robustness regions $B(x_i,r)$ and to other inputs.

\begin{table*}[t]
\caption{Results of generalization and defense against new adversarial attacks}
\label{tab:rq3}
\centering

\resizebox{\linewidth}{!}
{ 
\begin{tabular}{c|c|r|rrrrrrr|rrrrrrr|rrrrrrr}
\hline
\multirow{2}{*}{Model}                                                & \multirow{2}{*}{$r$} & \multicolumn{1}{c|}{\multirow{2}{*}{$n$}} & \multicolumn{7}{c|}{RGR/\%}                                                                                                                                                                     & \multicolumn{7}{c|}{DSR/\%}                                                                                                                                                                    & \multicolumn{7}{c}{DGSR/\%}                                                                                                                                                                   \\ \cline{4-24} 
                                                                      &                         & \multicolumn{1}{c|}{}                     & \multicolumn{1}{c}{CAR} & \multicolumn{1}{c}{PRD} & \multicolumn{1}{c}{APR} & \multicolumn{1}{c}{REA} & \multicolumn{1}{c}{TRA}  & \multicolumn{1}{c}{ART} & \multicolumn{1}{c|}{Ours} & \multicolumn{1}{c}{CAR} & \multicolumn{1}{c}{PRD} & \multicolumn{1}{c}{APR} & \multicolumn{1}{c}{REA} & \multicolumn{1}{c}{TRA} & \multicolumn{1}{c}{ART} & \multicolumn{1}{c|}{Ours} & \multicolumn{1}{c}{CAR} & \multicolumn{1}{c}{PRD} & \multicolumn{1}{c}{APR} & \multicolumn{1}{c}{REA} & \multicolumn{1}{c}{TRA} & \multicolumn{1}{c}{ART} & \multicolumn{1}{c}{Ours} \\ \hline
\multirow{15}{*}{\begin{tabular}[c]{@{}c@{}}FNN\_\\ small\end{tabular}} & \multirow{5}{*}{0.05}   & 50                                        & 7.6                      & 44.4                      & 47.2                      & 10.8                      & 80.0~          & 10.0~             & \textbf{100.0}            & 0.0                      & 0.0                       & 2.0                       & 0.0                     & 26.0~         & 6.0~              & \textbf{100.0}            & 8.9                      & 3.2                       & 2.3                       & 0.0                       & 9.4~          & 8.9~              & \textbf{91.3}            \\
                                                                      &                         & 100                                       & 2.2                      & 48.7                      & 46.7                      & 9.1                       & 0.0~           & 10.0~             & \textbf{100.0}            & 1.0                      & 0.0                       & 0.0                       & 0.0                       & 2.0~          & 20.0~             & \textbf{100.0}            & 9.3                      & 5.1                       & 0.8                       & 0.0                       & 9.8~          & 11.3~             & \textbf{96.6}            \\
                                                                      &                         & 200                                       & 5.0                      & 42.6                      & 16.1                      & 8.5                       & 0.0~           & 11.0~             & \textbf{100.0}            & 2.0                      & 0.0                       & 0.5                       & 0.0                       & 3.5~          & 10.0~             & \textbf{100.0}            & 8.6                      & 1.1                       & 1.3                       & 0.0                       & 9.8~          & 10.1~             & \textbf{96.6}            \\
                                                                      &                         & 500                                       & 6.6                      & 39.0                      & 10.3                      & 7.2                       & 0.0~           & 11.0~             & \textbf{100.0}            & 1.4                      & 0.0                       & 3.8                       & 0.0                       & 2.8~          & 9.6~              & \textbf{100.0}            & 8.3                      & 0.5                       & 5.3                       & 0.0                       & 9.8~          & 8.9~              & \textbf{96.6}            \\
                                                                      &                         & 1000                                      & 12.7                     & 41.3                      & 10.9                      & --                        & 2.0~           & 11.0~             & \textbf{100.0}            & 2.3                      & 0.0                       & 0.4                       & --                        & 3.9~          & 14.4~             & \textbf{100.0}            & 6.1                      & 0.2                       & 0.9                       & --                        & 9.8~          & 11.3~             & \textbf{96.6}            \\
                                                                      & \multirow{5}{*}{0.1}    & 50                                        & 0.8                      & 24.2                      & 32.0                      & 2.4                       & 20.0~          & 14.0~             & \textbf{100.0}            & 0.0                      & 0.0                       & 0.0                       & 0.0                       & 0.0~          & 8.0~              & \textbf{100.0}            & 0.0                      & 0.0                       & 0.0                       & 0.0                       & 9.8~          & 9.6~              & \textbf{96.3}            \\
                                                                      &                         & 100                                       & 1.7                      & 29.8                      & 21.9                      & 1.8                       & 10.0~          & 9.0~              & \textbf{100.0}            & 0.0                      & 0.0                       & 0.0                       & 0.0                       & 13.0~         & 9.0~              & \textbf{100.0}            & 0.0                      & 0.0                       & 0.0                       & 0.0                       & 9.8~          & 9.6~              & \textbf{95.0}            \\
                                                                      &                         & 200                                       & 1.1                      & 32.7                      & 12.9                      & 1.4                       & 5.0~           & 12.0~             & \textbf{100.0}            & 0.0                      & 0.0                       & 0.0                       & 0.0                       & 10.5~         & 11.5~             & \textbf{100.0}            & 0.0                      & 0.0                       & 0.0                       & 0.0                       & 9.8~          & 10.1~             & \textbf{95.2}            \\
                                                                      &                         & 500                                       & 1.7                      & 37.5                      & 17.3                      & 1.4                       & 8.0~           & 14.0~             & \textbf{100.0}            & 0.0                      & 0.0                       & 0.0                       & 0.0                       & 10.0~         & 10.0~             & \textbf{100.0}            & 0.0                      & 0.0                       & 0.1                       & 0.0                       & 9.8~          & 9.8~              & \textbf{95.3}            \\
                                                                      &                         & 1000                                      & 2.3                      & 42.5                      & 15.2                      & --                        & 13.0~          & 14.0~             & \textbf{100.0}            & 0.0                      & 0.0                       & 1.2                       & --                        & 10.0~         & 10.0~             & \textbf{100.0}            & 0.0                      & 0.0                       & 0.2                       & --                        & 9.8~          & 9.8~              & \textbf{96.5}            \\
                                                                      & \multirow{5}{*}{0.3}    & 50                                        & 0.0                      & 24.2                      & 37.4                      & 0.8                       & 20.0~          & 8.0~              & \textbf{100.0}             & 0.0                      & 0.0                       & 0.0                       & 0.0                      & 8.0~          & 8.0~              & \textbf{100.0}             & 0.0                      & 0.0                       & 0.0                       & 0.0                      & 9.8~          & 9.7~              & \textbf{96.5}            \\
                                                                      &                         & 100                                       & 7.7                      & 22.0                      & 18.4                      & 0.4                       & 10.0~          & 11.0~             & \textbf{100.0}             & 0.0                      & 0.0                       & 0.0                       & 0.0                      & 13.0~         & 11.0~             & \textbf{100.0}            & 0.0                      & 0.0                       & 0.0                       & 0.0                       & 9.8~          & 9.6~              & \textbf{96.5}            \\
                                                                      &                         & 200                                       & 0.0                      & 26.1                      & 8.8                       & 0.5                       & 5.0~           & 8.0~              & \textbf{100.0}             & 0.0                      & 0.0                       & 0.0                       & 0.0                      & 10.5~         & 9.5~              & \textbf{100.0}            & 0.0                      & 0.0                       & 0.0                       & 0.0                       & 9.8~          & 9.6~              & \textbf{96.6}            \\
                                                                      &                         & 500                                       & 0.4                      & 30.4                      & 12.8                      & 0.3                       & 8.0~           & 10.0~             & \textbf{100.0}             & 0.0                      & 0.0                       & 0.0                       & 0.0                      & 10.0~         & 10.4~             & \textbf{100.0}             & 0.0                      & 0.0                       & 0.0                       & 0.0                      & 9.8~          & 10.3~             & \textbf{96.6}            \\
                                                                      &                         & 1000                                      & 0.6                      & 31.9                      & 12.5                      & --                        & 13.0~          & 11.0~             & \textbf{100.0}             & 0.0                      & 0.0                       & 0.0                       & --                       & 9.8~          & 9.8~              & \textbf{100.0}             & 0.0                      & 0.0                       & 0.0                       & --                       & 9.8~          & 9.8~              & \textbf{96.6}            \\ \hline
\multirow{15}{*}{\begin{tabular}[c]{@{}c@{}}FNN\_\\ big\end{tabular}}   & \multirow{5}{*}{0.05}   & 50                                        & 6.6                      & 63.0                      & 64.8                      & 8.4                       & \textbf{100.0}~~ & 24.0~             & \textbf{100.0}            & 2.0                      & 2.0                       & 8.0                       & 0.0                   & 82.0~         & 24.0~             & \textbf{100.0}            & 37.7                     & 38.3                      & 31.7                      & 0.0                       & 53.5~         & 11.3~             & \textbf{97.2}            \\
                                                                      &                         & 100                                       & 10.9                     & 63.3                      & 67.4                      & 9.6                       & 90.0~          & 9.0~              & \textbf{100.0}            & 7.0                      & 1.0                       & 11.0                      & 0.0                       & 60.0~         & 9.0~              & \textbf{100.0}            & 35.7                     & 39.8                      & 19.0                      & 0.0                       & 54.9~         & 10.3~             & \textbf{97.2}            \\
                                                                      &                         & 200                                       & 10.7                     & 58.9                      & 64.6                      & 7.3                       & 69.5~          & 6.0~              & \textbf{100.0}            & 4.5                      & 2.5                       & 2.5                       & 0.0                       & 54.5~         & 5.0~              & \textbf{100.0}            & 37.9                     & 33.2                      & 12.8                      & 0.0                       & 57.0~         & 8.9~              & \textbf{97.2}            \\
                                                                      &                         & 500                                       & 10.6                     & 54.8                      & 66.2                      & --                        & 78.8~          & 10.0~             & \textbf{100.0}            & 4.8                      & 0.6                       & 0.0                       & --                        & 52.8~         & 7.4~              & \textbf{100.0}            & 35.4                     & 29.7                      & 1.7                       & --                        & 61.2~         & 10.1~             & \textbf{97.2}            \\
                                                                      &                         & 1000                                      & 11.8                     & 54.6                      & 68.4                      & --                        & 79.2~          & 20.0~             & \textbf{100.0}            & 7.1                      & 0.4                       & 0.0                       & --                        & 56.1~         & 19.2~             & \textbf{100.0}            & 36.8                     & 25.5                      & 1.8                       & --                        & 66.2~         & 11.3~             & \textbf{97.2}            \\
                                                                      & \multirow{5}{*}{0.1}    & 50                                        & 20.2                     & 52.6                      & 44.2                      & 3.4                       & 0.0~~           & 8.0~              & \textbf{100.0}            & 2.0                      & 0.0                       & 0.0                       & 0.0                      & 0.0~          & 8.0~              & \textbf{100.0}            & 2.6                      & 0.7                       & 1.1                       & 0.0                       & 1.8~          & 9.7~              & \textbf{97.2}            \\
                                                                      &                         & 100                                       & 2.0                      & 47.5                      & 46.3                      & 3.2                       & 20.0~          & 11.0~             & \textbf{100.0}            & 0.0                      & 0.0                       & 0.0                       & 0.0                       & 7.0~          & 11.0~             & \textbf{100.0}            & 1.8                      & 0.9                       & 0.0                       & 0.0                       & 4.0~          & 9.8~              & \textbf{97.2}            \\
                                                                      &                         & 200                                       & 19.2                     & 49.6                      & 46.8                      & 2.5                       & 22.5~          & 10.0~             & \textbf{100.0}            & 1.5                      & 0.0                       & 0.0                       & 0.0                       & 6.0~          & 10.5~             & \textbf{100.0}            & 2.6                      & 0.8                       & 0.1                       & 0.0                       & 4.9~          & 9.8~              & \textbf{97.2}            \\
                                                                      &                         & 500                                       & 1.2                      & 50.6                      & 49.3                      & --                        & 23.6~          & 15.0~             & \textbf{100.0}            & 0.0                      & 0.2                       & 0.0                       & --                        & 2.8~          & 9.0~              & \textbf{100.0}            & 1.2                      & 0.2                       & 0.0                       & --                        & 5.4~          & 8.9~              & \textbf{97.2}            \\
                                                                      &                         & 1000                                      & 1.0                      & 51.1                      & 49.9                      & --                        & 42.1~          & 11.0~             & \textbf{100.0}            & 0.0                      & 0.0                       & 0.0                       & --                        & 7.2~          & 10.5~             & \textbf{100.0}            & 1.4                      & 0.1                       & 0.0                       & --                        & 9.2~          & 10.3~             & \textbf{97.2}            \\
                                                                      & \multirow{5}{*}{0.3}    & 50                                        & 0.0                      & 42.2                      & 20.4                      & 2.8                       & 2.0~           & 14.0~             & \textbf{100.0}            & 0.0                      & 0.0                       & 0.0                       & 0.0                       & 0.0~          & 14.0~             & \textbf{100.0}            & 0.0                      & 0.0                       & 0.0                       & 0.0                       & 0.0~          & 11.3~             & \textbf{96.5}            \\
                                                                      &                         & 100                                       & 0.1                      & 43.4                      & 27.8                      & 1.4                       & 2.0~           & 14.0~             & \textbf{100.0}            & 0.0                      & 0.0                       & 0.0                       & 0.0                       & 0.0~          & 14.0~             & \textbf{100.0}            & 0.0                      & 0.0                       & 0.0                       & 0.0                       & 0.0~          & 11.3~             & \textbf{96.5}            \\
                                                                      &                         & 200                                       & 0.1                      & 40.8                      & 31.0                      & 1.6                       & 11.0~          & 10.0~             & \textbf{100.0}            & 0.0                      & 0.0                       & 0.0                       & 0.0                       & 0.0~          & 10.0~             & \textbf{100.0}            & 0.0                      & 0.0                       & 0.0                       & 0.0                       & 0.0~          & 10.3~             & \textbf{96.6}            \\
                                                                      &                         & 500                                       & 0.3                      & 36.9                      & 36.0                      & --                        & 9.2~           & 9.0~              & \textbf{100.0}            & 0.0                      & 0.0                       & 0.0                       & --                        & 0.0~          & 10.2~             & \textbf{100.0}            & 0.0                      & 0.0                       & 0.0                       & --                        & 0.0~          & 10.3~             & \textbf{97.2}            \\
                                                                      &                         & 1000                                      & 0.4                      & 36.5                      & 37.1                      & --                        & 15.0~          & 12.0~             & \textbf{100.0}            & 0.0                      & 0.0                       & 0.0                       & --                        & 0.0~          & 11.8~             & \textbf{100.0}             & 0.0                      & 0.0                       & 0.0                       & --                       & 0.0~          & 11.3~             & \textbf{97.2}            \\ \hline
\multirow{15}{*}{\begin{tabular}[c]{@{}c@{}}CNN\end{tabular}} & \multirow{5}{*}{0.05}   & 50                                        & 27.0                     & 69.6                      & 67.6                      & 47.2                      & \textbf{100.0}~ & 0.0~              & \textbf{100.0}            & 0.0                      & 0.0                       & 12.0                      & 0.0                              & 100.~        & 18.0~             & \textbf{100.0}            & 70.6                     & 37.5                      & 70.4                      & 0.0                        & 87.0~         & 9.8~              & \textbf{96.6}            \\
                                                                      &                         & 100                                       & 29.4                     & 74.5                      & 65.6                      & 45.7                      & \textbf{100.0}~ & 60.0~             & \textbf{100.0}            & 1.0                      & 0.0                       & 10.0                      & 0.0                      & 86.0~         & 27.0~             & \textbf{100.0}            & 70.6                     & 29.4                      & 69.2                      & 0.0                       & 88.2~         & 9.7~              & \textbf{98.3}            \\
                                                                      &                         & 200                                       & 27.6                     & 79.8                      & 62.5                      & 50.9                      & 97.0~          & 0.0~              & \textbf{100.0}            & 1.0                      & 0.0                       & 3.0                       & 0.0                       & 84.0~         & 20.5~             & \textbf{100.0}            & 70.6                     & 10.4                      & 70.6                      & 0.0                       & 89.9~         & 9.7~              & \textbf{98.3}            \\
                                                                      &                         & 500                                       & 30.8                     & 89.0                      & 62.3                      & --                        & 92.4~          & --~               & \textbf{100.0}            & 1.4                      & 0.0                       & 0.8                       & --                        & 87.0~         & --~               & \textbf{100.0}            & 70.6                       & 4.3                        & 71.4                      & --                     & 91.2~         & --~               & \textbf{98.3}            \\
                                                                      &                         & 1000                                      & 32.9                     & 92.2                      & 63.8                      & --                        & 90.8~          & --~               & \textbf{100.0}            & 1.5                      & 0.0                       & 0.2                       & --                        & 86.8~         & --~               & \textbf{100.0}            & 70.7                       & 0.9                        & 65.7                      & --                     & 92.1~         & --~               & \textbf{98.3}            \\
                                                                      & \multirow{5}{*}{0.1}    & 50                                        & 12.2                     & 53.0                      & 41.0                      & 10.4                      & 80.0~          & 0.0~              & \textbf{100.0}            & 2.0                      & 0.0                       & 2.0                       & 0.0                       & 50.0~         & 8.0~              & \textbf{100.0}            & 8.3                      & 0.0                       & 8.7                       & 0.0                       & 44.9~         & 10.3~             & \textbf{98.2}            \\
                                                                      &                         & 100                                       & 22.2                     & 69.1                      & 46.0                      & 14.7                      & 88.0~          & 10.0~             & \textbf{100.0}            & 1.0                      & 0.0                       & 0.0                       & 0.0                       & 57.0~         & 11.0~             & \textbf{100.0}            & 8.2                      & 0.0                       & 8.6                       & 0.0                       & 48.5~         & 9.6~              & \textbf{98.2}            \\
                                                                      &                         & 200                                       & 21.8                     & 73.2                      & 46.1                      & 13.0                      & 96.5          & 20.0~             & \textbf{100.0}            & 0.5                      & 0.0                       & 0.5                       & 0.0                        & 62.0~         & 11.0~             & \textbf{100.0}            & 8.2                      & 0.0                       & 9.2                       & 0.0                       & 53.3~         & 9.8~              & \textbf{98.2}            \\
                                                                      &                         & 500                                       & 30.4                     & 84.5                      & 47.0                      & --                        & 93.2~          & --~               & \textbf{100.0}            & 0.2                      & 0.0                       & 0.0                       & --                        & 66.6~         & --~               & \textbf{100.0}             & 8.2                       & 0.0                        & 8.4                       & --                     & 62.4~         & --~               & \textbf{98.2}            \\
                                                                      &                         & 1000                                      & 24.6                     & 90.3                      & 48.7                      & --                        & 94.7~          & --~               & \textbf{100.0}            & 0.6                      & 0.0                       & 0.0                       & --                        & 70.1~         & --~               & \textbf{100.0}             & 8.3                       & 0.0                        & 4.6                       & --                     & 67.8~         & --~               & \textbf{98.3}            \\
                                                                      & \multirow{5}{*}{0.3}    & 50                                        & 0.0                      & 44.8                      & 70.0                      & 4.0                       & 0.0~~           & 10.0~             & \textbf{100.0}            & 0.0                      & 0.0                       & 0.0                       & 0.0                      & 0.0~          & 12.0~             & \textbf{100.0}            & 0.0                      & 0.0                       & 0.0                       & 0.0                       & 0.0~          & 9.8~              & \textbf{98.3}            \\
                                                                      &                         & 100                                       & 0.0                      & 59.9                      & 75.5                      & 2.7                       & 13.0~          & 10.0~             & \textbf{100.0}            & 0.0                      & 0.0                       & 0.0                       & 0.0                       & 0.0~          & 7.0~              & \textbf{100.0}            & 0.0                      & 0.0                       & 0.0                       & 0.0                       & 0.0~          & 8.9~              & \textbf{98.3}            \\
                                                                      &                         & 200                                       & 0.0                      & 69.4                      & 78.2                      & 3.6                       & 12.5~          & 0.0~              & \textbf{100.0}            & 0.0                      & 0.0                       & 0.0                       & 0.0                       & 0.0~          & 7.0~              & \textbf{100.0}            & 0.0                      & 0.0                       & 0.0                       & 0.0                       & 0.0~          & 8.9~              & \textbf{98.3}            \\
                                                                      &                         & 500                                       & 1.3                      & 79.6                      & 73.3                      & --                        & 12.6~          & --~               & \textbf{100.0}            & 0.0                      & 0.0                       & 0.0                       & --                        & 0.0~          & --~               & \textbf{100.0}            & 0.0                      & 0.0                        & 0.0                       & --                       & 0.0~          & --~               & \textbf{98.3}            \\
                                                                      &                         & 1000                                      & 1.1                      & 85.9                      & 70.6                      & --                        & 16.0~          & --~               & \textbf{100.0}            & 0.0                      & 0.0                       & 0.0                       & --                        & 0.0~          & --~               & \textbf{100.0}            & 0.0                       & 0.0                        & 0.0                       & --                      & 0.0~          & --~               & \textbf{98.2}            \\ \hline
\end{tabular}
}

\vspace{-1em}
\end{table*}

\subsection{Repair performance }

The results of the local robustness repair task on MNIST are summarized in Table~\ref{tab:rq1-2}.
We evaluate overall repair performance with $\mathrm{RSR}$, $\mathrm{DD}$, and runtime, representing fix success, accuracy retention, and speed, respectively. 
PRDNN, REASSURE, and \ours reach 100\% $\mathrm{RSR}$ due to their provable designs. APRNN, despite being provable, underperforms on some FNN\_small cases, while ART struggles with high-dimensional inputs. CARE and TRADES, lacking provable guarantees, show inferior $\mathrm{RSR}$ compared to provable methods. CARE's suboptimal repair across all models may stem from the complex interplay between local robustness repair and DNNs' numerous parameters, making it difficult to pinpoint and modify relevant parts. Although TRADES performs better on CNNs, its effectiveness wanes with larger-radius local robustness errors. In terms of $\mathrm{DD}$, \ours and REASSURE consistently outperform alternative methods across all models.

The results of correcting safety property violations on ACAS Xu are presented in Table~\ref{tab:acasxu}. Each patch module here is of the same size as the original network.
\ours, PRDNN, REASSURE, and ART achieve a perfect 100\% $\mathrm{RSR}$ with provable repairs. Both ART and our method showcase remarkable $\mathrm{RGR}$ scores, achieving a perfect 100\% among all the evaluated tools. \ours also excels with a $\mathrm{FDD}$ of zero, outperforming CARE, ART, PRDNN, and REASSURE. These results emphasize our method's strength in delivering provable repairs and preserving high-fidelity functionality.



\begin{table*}[t]
\caption{Results of repairing local robustness on CIFAR-10}
\label{tab:cifar10}
\centering

\resizebox{\linewidth}{!}
{
\begin{tabular}{cc|rrrrrrrrrr|rrrrrrrrrr}
\hline
\multicolumn{2}{l|}{\multirow{3}{*}{}}                 & \multicolumn{10}{c|}{VGG19}                                                                                                                                                                                                                                                   & \multicolumn{10}{c}{ResNet-18}                                                                                                                                                                                                                                               \\ \cline{3-22} 
\multicolumn{2}{l|}{}                                  & \multicolumn{5}{c|}{$r = 4 /  255$}                                                                                                                     & \multicolumn{5}{c|}{$r = 8/255$}                                                                                                           & \multicolumn{5}{c|}{$r = 4/255$}                                                                                                                     & \multicolumn{5}{c}{$r = 8/255$}                                                                                                           \\ \cline{3-22} 
\multicolumn{2}{l|}{}                                  & \multicolumn{1}{c}{50} & \multicolumn{1}{c}{100} & \multicolumn{1}{c}{200} & \multicolumn{1}{c}{500} & \multicolumn{1}{c|}{1000}           & \multicolumn{1}{c}{50} & \multicolumn{1}{c}{100} & \multicolumn{1}{c}{200} & \multicolumn{1}{c}{500} & \multicolumn{1}{c|}{1000} & \multicolumn{1}{c}{50} & \multicolumn{1}{c}{100} & \multicolumn{1}{c}{200} & \multicolumn{1}{c}{500} & \multicolumn{1}{c|}{1000}           & \multicolumn{1}{c}{50} & \multicolumn{1}{c}{100} & \multicolumn{1}{c}{200} & \multicolumn{1}{c}{500} & \multicolumn{1}{c}{1000} \\ \hline
\multicolumn{1}{c|}{\multirow{5}{*}{RSR/\%}}  & CARE   & 8.0                    & 5.0                     & 4.0                     & 11.0                    & \multicolumn{1}{r|}{10.9}           & 2.0                    & 4.0                     & 2.0                     & 2.2                     & 2.4                       & 0.0                    & 0.0                     & 0.0                     & 0.0                     & \multicolumn{1}{r|}{0.0}            & 0.0                    & 0.0                     & 0.0                     & 0.0                     & 0.0                      \\
\multicolumn{1}{c|}{}                         & PRDNN  & \textbf{100.0}         & \textbf{100.0}          & \textbf{100.0}          & --                      & \multicolumn{1}{r|}{--}             & \textbf{100.0}                  & \textbf{100.0}                   & \textbf{100.0}                   & --                      & --                        & \textbf{100.0}                  & \textbf{100.0}                   & \textbf{100.0}                   & --                      & \multicolumn{1}{r|}{--}             & \textbf{100.0}                  & \textbf{100.0}                   & \textbf{100.0}                   & --                      & --                       \\
\multicolumn{1}{c|}{}                         & TRADE & 98.0                   & 95.0                    & 97.0                    & 95.0                    & \multicolumn{1}{r|}{96.4}           & \textbf{100.0}                  & 97.0                    & 93.5                    & 95.6                    & 95.8                      & \textbf{100.0}                  & 99.0                    & \textbf{100.0}                   & 99.8                    & \multicolumn{1}{r|}{99.5}           & \textbf{100.0}                  & \textbf{100.0}                    & \textbf{100.0}                    & 99.8                    & 96.8                     \\
\multicolumn{1}{c|}{}                         & APRNN  & \textbf{100.0}         & \textbf{100.0}          & \textbf{100.0}          & \textbf{100.0}          & \multicolumn{1}{r|}{\textbf{100.0}} & \textbf{100.0}         & \textbf{100.0}          & \textbf{100.0}          & \textbf{100.0}          & \textbf{100.0}            & \textbf{100.0}         & \textbf{100.0}          & \textbf{100.0}          & \textbf{100.0}          & \multicolumn{1}{r|}{\textbf{100.0}} & \textbf{100.0}         & \textbf{100.0}          & \textbf{100.0}          & \textbf{100.0}          & \textbf{100.0}           \\
\multicolumn{1}{c|}{}                         & Ours   & \textbf{100.0}         & \textbf{100.0}          & \textbf{100.0}          & \textbf{100.0}          & \multicolumn{1}{r|}{\textbf{100.0}} & \textbf{100.0}         & \textbf{100.0}          & \textbf{100.0}          & \textbf{100.0}          & \textbf{100.0}            & \textbf{100.0}         & \textbf{100.0}          & \textbf{100.0}          & \textbf{100.0}          & \multicolumn{1}{r|}{\textbf{100.0}} & \textbf{100.0}         & \textbf{100.0}          & \textbf{100.0}          & \textbf{100.0}          & \textbf{100.0}           \\ \hline
\multicolumn{1}{c|}{\multirow{5}{*}{RGR/\%}}  & CARE   & 6.6                    & 4.2                     & 3.6                     & 11.7                    & \multicolumn{1}{r|}{11.8}           & 2.0                    & 3.6                     & 2.7                     & 2.4                     & 2.9                       & 2.6                    & 2.7                     & 9.1                     & 12.2                    & \multicolumn{1}{r|}{6.7}            & 3.8                    & 2.6                     & 2.3                     & 2.7                     & 2.3                      \\
\multicolumn{1}{c|}{}                         & PRDNN  & 69.8                   & 68.2                    & 68.5                    & --                      & \multicolumn{1}{r|}{--}             & 58.0                   & 66.6                    & 61.4                    & --                      & --                        & 68.0                   & 68.6                    & 67.2                    & --                      & \multicolumn{1}{r|}{--}             & 59.8                   & 57.8                    & 54.9                    & --                      & --                       \\
\multicolumn{1}{c|}{}                         & TRADE & \textbf{100.0}                  & 86.0                    & 90.0                    & 97.6                    & \multicolumn{1}{r|}{97.0}           & \textbf{100.0}                   & \textbf{100.0}                    & 95.0                   & 98.0                    & 98.9                      & \textbf{100.0}                  & \textbf{100.0}                    & \textbf{100.0}                   & \textbf{100.0}                   & \multicolumn{1}{r|}{99.1}           & \textbf{100.0}                  & \textbf{100.0}                    & \textbf{100.0}                   & \textbf{100.0}                   & 97.3                     \\
\multicolumn{1}{c|}{}                         & APRNN  & 69.2                   & 68.9                    & 70.9                    & 75.5                    & \multicolumn{1}{r|}{80.4}           & 63.2                   & 65.2                    & 68.2                    & 66.1                    & 70.7                      & 70.6                   & 65.1                    & 62.0                    & 62.7                    & \multicolumn{1}{r|}{68.2}           & 65.0                   & 55.6                    & 53.0                    & 49.6                    & 59.0                     \\
\multicolumn{1}{c|}{}                         & Ours   & \textbf{100.0}         & \textbf{100.0}          & \textbf{100.0}          & \textbf{100.0}          & \multicolumn{1}{r|}{\textbf{100.0}} & \textbf{100.0}         & \textbf{100.0}          & \textbf{100.0}          & \textbf{100.0}          & \textbf{100.0}            & \textbf{100.0}         & \textbf{100.0}          & \textbf{100.0}          & \textbf{100.0}          & \multicolumn{1}{r|}{\textbf{100.0}} & \textbf{100.0}         & \textbf{100.0}          & \textbf{100.0}          & \textbf{100.0}          & \textbf{100.0}           \\ \hline
\multicolumn{1}{c|}{\multirow{5}{*}{DD/\%}}   & CARE   & 0.1                    & 0.2                     & 0.2                     & 0.2                     & \multicolumn{1}{r|}{0.3}            & 0.1                    & 0.2                     & 0.2                     & 0.1                     & 0.3                       & 0.1                    & 0.1                     & 0.2                     & 0.3                     & \multicolumn{1}{r|}{0.2}            & 0.1                    & 0.0                     & 0.0                     & 0.1                     & 0.1                      \\
\multicolumn{1}{c|}{}                         & PRDNN  & 13.6                   & 24.8                    & 30.3                    & --                      & \multicolumn{1}{r|}{--}             & 19.8                   & 40.5                    & 50.2                    & --                      & --                        & 2.1                    & 2.1                     & 2.0                     & --                      & \multicolumn{1}{r|}{--}             & 14.0                   & 13.7                    & 12.2                    & --                      & --                       \\
\multicolumn{1}{c|}{}                         & TRADE & 18.6                   & 11.9                    & 7.7                    & 6.3                    & \multicolumn{1}{r|}{6.1}           & 25.8                   & 18.5                    & 11.5                    & 8.5                    & 8.3                     & 43.4                   & 29.2                    & 20.1                    & 14.4                    & \multicolumn{1}{r|}{10.1}          & 38.2                   & 34.9                    & 15.9                    & 11.9                    & 10.1                     \\
\multicolumn{1}{c|}{}                         & APRNN  & 14.2                   & 20.0                    & 29.3                    & 48.7                    & \multicolumn{1}{r|}{46.4}           & 24.0                   & 39.5                    & 61.9                    & 75.5                    & 82.1                      & 1.8                    & 1.8                     & 2.2                     & 3.1                     & \multicolumn{1}{r|}{6.1}            & 5.5                    & 4.2                     & 7.6                     & 12.7                    & 19.5                     \\
\multicolumn{1}{c|}{}                         & Ours   & \textbf{0.0}           & \textbf{0.0}            & \textbf{0.0}            & \textbf{0.0}            & \multicolumn{1}{r|}{\textbf{0.0}}   & \textbf{0.0}           & \textbf{0.0}            & \textbf{0.0}            & \textbf{0.0}            & \textbf{0.0}              & \textbf{0.0}           & \textbf{0.0}            & \textbf{0.0}            & \textbf{0.0}            & \multicolumn{1}{r|}{\textbf{0.0}}   & \textbf{0.0}           & \textbf{0.0}            & \textbf{0.0}            & \textbf{0.0}            & \textbf{0.0}            \\ \hline
\multicolumn{1}{c|}{\multirow{5}{*}{DSR/\%}}  & CARE   & 0.0                    & 0.0                     & 0.0                     & 0.0                     & \multicolumn{1}{r|}{0.0}            & 0.0                    & 0.0                     & 0.0                     & 0.0                     & 0.0                       & 0.0                    & 0.0                     & 0.0                     & 0.0                     & \multicolumn{1}{r|}{0.0}            & 0.0                    & 0.0                     & 0.0                     & 0.0                     & 0.0                      \\
\multicolumn{1}{c|}{}                         & PRDNN  & 0.0                    & 0.0                     & 0.0                     & 0.0                     & \multicolumn{1}{r|}{0.0}            & 0.0                    & 0.0                     & 0.0                     & 0.0                     & 0.0                       & 0.0                    & 0.0                     & 0.0                     & 0.0                     & \multicolumn{1}{r|}{0.0}            & 0.0                    & 0.0                     & 0.0                     & 0.0                     & 0.0                      \\
\multicolumn{1}{c|}{}                         & APRNN  & 0.0                    & 0.0                     & 0.0                     & 0.0                     & \multicolumn{1}{r|}{0.0}            & 0.0                    & 0.0                     & 0.0                     & 0.0                     & 0.0                       & 0.0                    & 0.0                     & 0.0                     & 0.0                     & \multicolumn{1}{r|}{0.0}            & 0.0                    & 0.0                     & 0.0                     & 0.0                     & 0.0                      \\
\multicolumn{1}{c|}{}                         & TRADE & 0.0                    & 0.0                     & 0.0                     & 0.0                     & \multicolumn{1}{r|}{0.0}            & 0.0                    & 0.0                     & 0.0                     & 0.0                     & 0.0                       & 0.0                    & 0.0                     & 0.0                     & 0.0                     & \multicolumn{1}{r|}{0.0}            & 0.0                    & 0.0                     & 0.0                     & 0.0                     & 0.0                      \\
\multicolumn{1}{c|}{}                         & Ours   & \textbf{100.0}         & \textbf{100.0}          & \textbf{100.0}          & \textbf{100.0}          & \multicolumn{1}{r|}{\textbf{100.0}} & \textbf{100.0}         & \textbf{100.0}          & \textbf{100.0}          & \textbf{100.0}          & \textbf{100.0}            & \textbf{100.0}         & \textbf{100.0}          & \textbf{100.0}          & \textbf{100.0}          & \multicolumn{1}{r|}{\textbf{100.0}} & \textbf{100.0}         & \textbf{100.0}          & \textbf{100.0}          & \textbf{100.0}          & \textbf{100.0}           \\ \hline
\multicolumn{1}{c|}{\multirow{5}{*}{DGSR/\%}} & CARE   & 0.0                    & 0.0                     & 0.0                     & 0.0                     & \multicolumn{1}{r|}{0.0}            & 0.0                    & 0.0                     & 0.0                     & 0.0                     & 0.0                       & 0.0                    & 0.0                     & 0.0                     & 0.0                     & \multicolumn{1}{r|}{0.0}            & 0.0                    & 0.0                     & 0.0                     & 0.0                     & 0.0                      \\
\multicolumn{1}{c|}{}                         & PRDNN  & 0.0                    & 0.0                     & 0.0                     & --                      & \multicolumn{1}{r|}{--}             & 0.0                    & 0.0                     & 0.0                     & --                      & --                        & 0.0                    & 0.0                     & 0.0                     & --                      & \multicolumn{1}{r|}{--}             & 0.0                    & 0.0                     & 0.0                     & --                      & --                       \\
\multicolumn{1}{c|}{}                         & APRNN  & 0.0                    & 0.0                     & 0.0                     & 0.0                     & \multicolumn{1}{r|}{0.0}            & 0.0                    & 0.0                     & 0.0                     & 0.0                     & 0.0                       & 0.0                    & 0.0                     & 0.0                     & 0.0                     & \multicolumn{1}{r|}{0.0}            & 0.0                    & 0.0                     & 0.0                     & 0.0                     & 0.0                      \\
\multicolumn{1}{c|}{}                         & TRADE & 3.6                    & 3.7                     & 2.9                     & 3.4                     & \multicolumn{1}{r|}{3.4}            & 4.1                    & 3.7                     & 3.2                     & 2.3                     & 2.6                       & 5.7                    & 6.7                     & 6.0                     & 5.4                     & \multicolumn{1}{r|}{5.2}            & 5.6                    & 5.0                     & 5.4                     & 5.8                     & 5.2                      \\
\multicolumn{1}{c|}{}                         & Ours   & \textbf{93.4}          & \textbf{93.4}           & \textbf{93.4}           & \textbf{93.4}           & \multicolumn{1}{r|}{\textbf{93.4}}  & \textbf{93.4}          & \textbf{93.4}           & \textbf{93.4}           & \textbf{93.4}           & \textbf{93.4}             & \textbf{88.3}          & \textbf{88.3}           & \textbf{88.3}           & \textbf{88.3}           & \multicolumn{1}{r|}{\textbf{88.3}}  & \textbf{88.3}          & \textbf{88.3}           & \textbf{88.3}           & \textbf{88.3}           & \textbf{88.3}            \\ \hline
\multicolumn{1}{c|}{\multirow{5}{*}{Time/s}}  & CARE   & 47.8                   & 64.4                    & 106.5                   & 380.4                   & \multicolumn{1}{r|}{928.2}          & 36.9                   & 48.2                    & 96.0                    & 234.1                   & 443.0                     & 45.5                   & 58.9                    & 105.5                   & 243.5                   & \multicolumn{1}{r|}{384.9}          & 37.1                   & 75.4                    & 101.6                   & 180.7                   & 285.7                    \\
\multicolumn{1}{c|}{}                         & PRDNN  & 1731.8                 & 2640.6                  & 6307.4                  & --                      & \multicolumn{1}{r|}{--}             & 1160.4                 & 2210.2                  & 5126.9                  & --                      & --                        & 731.8                  & 1794.9                  & 3307.5                  & --                      & \multicolumn{1}{r|}{--}             & 715.4                  & 1300.2                  & 5757.3                  & --                      & --                       \\
\multicolumn{1}{c|}{}                         & TRADE & 68.7                   & 114.8                   & 194.6                   & 415.4                   & \multicolumn{1}{r|}{814.6}          & 67.2                   & 111.1                   & 187.2                   & 417.6                   & 806.8                     & 57.1                   & 107.8                   & 186.7                   & 405.0                   & \multicolumn{1}{r|}{794.3}          & 57.9                   & 107.8                   & 184.1                   & 408.1                   & 790.6                    \\
\multicolumn{1}{c|}{}                         & APRNN  & 188.9                  & 288.2                   & 451.3                   & 1645.6                  & \multicolumn{1}{r|}{6211.4}         & 209.2                  & 318.6                   & 492.4                   & 1750.4                  & 7787.1                    & 307.7                  & 746.1                   & 812.2                   & 1338.1                  & \multicolumn{1}{r|}{1710.8}         & 580.9                  & 990.8                   & 863.2                   & 1556.0                  & 1988.4                   \\
\multicolumn{1}{c|}{}                         & Ours   & 238.7                  & 471.6                   & 963.0                   & 2516.9                  & \multicolumn{1}{r|}{5441.6}         & 229.9                  & 470.7                   & 950.1                   & 2485.0                  & 5436.0                    & 265.8                  & 525.1                   & 1050.3                  & 2586.4                  & \multicolumn{1}{r|}{5262.3}         & 264.7                  & 523.1                   & 1042.8                  & 2627.8                  & 5325.1                   \\ \hline
\end{tabular}
}
\end{table*}

\subsection{Generalization}
In the experiment, we assess the repaired network's generalization and resistance to new adversarial attacks (Table~\ref{tab:rq3}). \ours leads in $\mathrm{RGR}$, demonstrating superior performance compared to other baselines. PRDNN, APRNN, and TRADES also show some degree of generalization, particularly on CNNs. For $\mathrm{DSR}$, \ours outperforms competitors, with TRADES showing moderate defense against small-radius attacks on FNN\_big and CNN. Evaluating $\mathrm{DGSR}$, \ours again tops the list, while CARE, PRDNN, APRNN, and TRADES show limited defense against FNN\_big and CNN. $\mathrm{DGSR}$ is key for gauging generalization to global inputs . Overall, these findings underscore \ours' breakthrough in repairing adversarial attacks.

\noindent\doublebox{
    \begin{minipage}{0.95\linewidth}
    {\bf Answer to RQ1 and RQ2:} 
    \ours consistently outperforms the baselines in both local robustness repair and safety property violation correction tasks. 
    It achieves 100\% repair success rate coupled with 0\% drawdown and acceptable efficiency.
    Additionally, it demonstrates significant generalization against unforeseen adversarial attacks.
    
    \end{minipage}
}

\subsection{Scalability Evaluation}
We examine \ours' scalability on VGG19 and ResNet18 for CIFAR-10 local robustness repair via feature-space repair. Results in Table~\ref{tab:cifar10} show \ours, PRDNN, and APRNN achieve 100\% $\mathrm{RSR}$, while TRADES follows closely. 
\ours leads in $\mathrm{RGR}$ and $\mathrm{DD}$ with TRADES following suit. 
\ours also significantly outperforms others in $\mathrm{DSR}$ and  matches \textit{accuracy levels}  in $\mathrm{DGSR}$, surpassing all tools. On larger datasets like Tiny ImageNet, \ours maintains superiority in repairing WRN101-2 and ResNet152, as seen in Table~\ref{tab:tinyimagenet}. \ours' running time is reasonable given its superior repair and generalization abilities. Although TRADES performs comparably, \ours consistently delivers higher performance across various metrics. 
For both datasets, patch modules are applied at the network's second-to-last layer, using a fully connected network with a single linear layer. The linear layer takes the output from the network's penultimate layer as input, and its output is added to the original network's output.

\noindent\doublebox{
    \begin{minipage}{0.95\linewidth}
    {\bf Answer to RQ3:} 
    \ours demonstrates good scalability in repairing local robustness on large-scale DNNs,
    which also achieves a 100\% repair success rate coupled with 0\% drawdown,
    outperforming the other tools.
    \end{minipage}
}

\begin{table*}[t]
\caption{Results of repairing local robustness on Tiny ImageNet}
\label{tab:tinyimagenet}
\centering
\resizebox{\linewidth}{!}
{
\begin{tabular}{c|c|r|rrr|rrr|rrr|rrr|rrr|rrr}
\hline
\multirow{2}{*}{Model}                                                & \multirow{2}{*}{$r$}     & \multicolumn{1}{c|}{\multirow{2}{*}{$n$}} & \multicolumn{3}{c|}{RSR/\%}                                                       & \multicolumn{3}{c|}{RGR/\%}                                                       & \multicolumn{3}{c|}{DD/\%}                                                       & \multicolumn{3}{c|}{DSR/\%}                                                       & \multicolumn{3}{c|}{DGSR/\%}                                                       & \multicolumn{3}{c}{Time/s}                                                      \\ \cline{4-21} 
                                                                      &                        & \multicolumn{1}{c|}{}                   & \multicolumn{1}{c}{CARE} & \multicolumn{1}{c}{TRADE}  & \multicolumn{1}{c|}{Ours} & \multicolumn{1}{c}{CARE} & \multicolumn{1}{c}{TRADE}  & \multicolumn{1}{c|}{Ours} & \multicolumn{1}{c}{CARE} & \multicolumn{1}{c}{TRADE} & \multicolumn{1}{c|}{Ours} & \multicolumn{1}{c}{CARE} & \multicolumn{1}{c}{TRADE}  & \multicolumn{1}{c|}{Ours} & \multicolumn{1}{c}{CARE} & \multicolumn{1}{c}{TRADE}   & \multicolumn{1}{c|}{Ours} & \multicolumn{1}{c}{CARE} & \multicolumn{1}{c}{TRADE} & \multicolumn{1}{c}{Ours} \\ \hline
\multirow{4}{*}{\begin{tabular}[c]{@{}c@{}}WRN\\ 101-2\end{tabular}}  & \multirow{2}{*}{$\frac{2}{255}$} & 500                                     & 0.0~~          & \textbf{100.0~~} & \textbf{100.0}            & 1.2~~          & \textbf{100.0~~} & \textbf{100.0}            & 0.1~~          & 46.8~~          & \textbf{0.0}              & 0.0~~          & \textbf{100.0~~} & \textbf{100.0}            & 0.0~~          & 6.6~~~~ & \textbf{64.4}             & 768.7~~        & 5040.8~              & 2029.6                   \\
                                                                      &                        & 1000                                    & 0.0~~          & \textbf{100.0~~} & \textbf{100.0}            & 0.5~~          & \textbf{100.0~~} & \textbf{100.0}            & \textbf{0.0~~} & 43.2~~          & \textbf{0.0}              & 0.0~~          & \textbf{100.0~~} & \textbf{100.0}            & 0.0~~          & 7.3~~~~ & \textbf{64.4}             & 1345.7~~       & 7088.6~              & 3260.7                   \\
                                                                      & \multirow{2}{*}{$\frac{4}{255}$} & 500                                     & 0.0~~          & \textbf{100.0~~} & \textbf{100.0}            & 0.0~~          & \textbf{100.0~~} & \textbf{100.0}            & \textbf{0.0~~} & 51.8~~          & \textbf{0.0}              & 0.0~~          & 97.6~~           & \textbf{100.0}            & 0.0~~          & 1.8~~~~ & \textbf{64.4}             & 721.9~~        & 4985.5~              & 1731.8                   \\
                                                                      &                        & 1000                                    & 0.0~~          & \textbf{100.0~~} & \textbf{100.0}            & 0.0~~          & \textbf{100.0~~} & \textbf{100.0}            & 0.1~~          & 49.5~~          & \textbf{0.0}              & 0.0~~          & 97.2~~           & \textbf{100.0}            & 0.0~~          & 2.1~~~~ & \textbf{64.4}             & 1440.8~~       & 6976.2~              & 3221.2                   \\ \hline
\multirow{4}{*}{\begin{tabular}[c]{@{}c@{}}ResNet\\ 152\end{tabular}} & \multirow{2}{*}{$\frac{2}{255}$} & 500                                     & 0.0~~          & \textbf{100.0~~} & \textbf{100.0}            & 0.0~~          & \textbf{100.0~~} & \textbf{100.0}            & \textbf{0.0~~} & 64.5~~          & \textbf{0.0}              & 0.0~~          & \textbf{100.0~~} & \textbf{100.0}            & 0.0~~          & 5.2~~~~ & \textbf{68.2}             & 545.3~~        & 4667.7~              & 2219.6                   \\
                                                                      &                        & 1000                                    & 0.0~~          & \textbf{100.0~~} & \textbf{100.0}            & 0.0~~          & \textbf{100.0~~} & \textbf{100.0}            & \textbf{0.0~~} & 62.6~~          & \textbf{0.0}              & 0.0~~          & 99.9~~           & \textbf{100.0}            & 0.0~~          & 6.1~~~~ & \textbf{68.2}             & 1092.9~~       & 7503.5~              & 4703.7                   \\
                                                                      & \multirow{2}{*}{$\frac{4}{255}$} & 500                                     & 0.0~~          & \textbf{100.0~~} & \textbf{100.0}            & 0.0~~          & \textbf{100.0~~} & \textbf{100.0}            & \textbf{0.0~~} & 64.5~~          & \textbf{0.0}              & 0.0~~          & 97.6~~           & \textbf{100.0}            & 0.0~~          & 1.3~~~~ & \textbf{68.2}             & 535.4~~        & 4624.2~              & 2336.1                   \\
                                                                      &                        & 1000                                    & 0.0~~          & \textbf{100.0~~} & \textbf{100.0}            & 0.0~~          & \textbf{100.0~~} & \textbf{100.0}            & \textbf{0.0~~} & 63.6~~          & \textbf{0.0}              & 0.0~~          & 98.5~~           & \textbf{100.0}            & 0.0~~          & 2.0~~~~ & \textbf{68.2}             & 1028.1~~       & 7499.9~              & 4366.8                   \\ \hline
\end{tabular}
}
\end{table*}

\begin{table*}
\caption{The efficacy of patch module scale}
\label{tab:rq4}
\centering
\resizebox{\linewidth}{!}{
\begin{tabular}{cc|rrrrrrrrrr|rrrrrrrrrr}
\hline
\multicolumn{2}{l|}{\multirow{3}{*}{}}              & \multicolumn{10}{c|}{VGG19}                                                                                                                                                                                                                                                                                                                                      & \multicolumn{10}{c}{ResNet-18}                                                                                                                                                                                                                                                                                                                                  \\ \cline{3-22} 
\multicolumn{2}{l|}{}                               & \multicolumn{5}{c|}{$r = 4 /  255$}                                                                                                                                                       & \multicolumn{5}{c|}{$r = 8 /  255$}                                                                                                                                                    & \multicolumn{5}{c|}{$r = 4 /  255$}                                                                                                                                                       & \multicolumn{5}{c}{$r = 8 /  255$}                                                                                                                                                    \\ \cline{3-22} 
\multicolumn{2}{l|}{}                               & \multicolumn{1}{c}{50}           & \multicolumn{1}{c}{100}          & \multicolumn{1}{c}{200}          & \multicolumn{1}{c}{500}          & \multicolumn{1}{c|}{1000}            & \multicolumn{1}{c}{50}           & \multicolumn{1}{c}{100}          & \multicolumn{1}{c}{200}          & \multicolumn{1}{c}{500}          & \multicolumn{1}{c|}{1000}         & \multicolumn{1}{c}{50}           & \multicolumn{1}{c}{100}          & \multicolumn{1}{c}{200}          & \multicolumn{1}{c}{500}          & \multicolumn{1}{c|}{1000}            & \multicolumn{1}{c}{50}           & \multicolumn{1}{c}{100}          & \multicolumn{1}{c}{200}          & \multicolumn{1}{c}{500}          & \multicolumn{1}{c}{1000}         \\ \hline
\multicolumn{1}{c|}{\multirow{2}{*}{DSR/\%}}   & PS & \textbf{100.0}                   & \textbf{100.0}                   & \textbf{100.0}                   & \textbf{100.0}                   & \multicolumn{1}{r|}{\textbf{100.0}}  & \textbf{100.0}                   & \textbf{100.0}                   & \textbf{100.0}                   & \textbf{100.0}                   & \textbf{100.0}                    & \textbf{100.0}                   & \textbf{100.0}                   & \textbf{100.0}                   & \textbf{100.0}                   & \multicolumn{1}{r|}{\textbf{100.0}}  & \textbf{100.0}                   & \textbf{100.0}                   & \textbf{100.0}                   & \textbf{100.0}                   & \textbf{100.0}                   \\
\multicolumn{1}{c|}{}                          & PL & \textbf{100.0}                   & \textbf{100.0}                   & \textbf{100.0}                   & \textbf{100.0}                   & \multicolumn{1}{r|}{\textbf{100.0}}  & \textbf{100.0}                   & \textbf{100.0}                   & \textbf{100.0}                   & \textbf{100.0}                   & \textbf{100.0}                    & 98.0                             & 99.0                             & 99.5                             & 99.2                             & \multicolumn{1}{r|}{99.4}            & \textbf{100.0}                   & \textbf{100.0}                   & 99.5                             & 99.8                             & 99.8                             \\ \hline
\multicolumn{1}{c|}{\multirow{2}{*}{DGSR/\%}}  & PS & \textbf{93.4}                    & \textbf{93.4}                    & \textbf{93.4}                    & \textbf{93.4}                    & \multicolumn{1}{r|}{\textbf{93.4}}   & \textbf{93.4}                    & \textbf{93.4}                    & \textbf{93.4}                    & \textbf{93.4}                    & \textbf{93.4}                     & \textbf{88.3}                    & \textbf{88.3}                    & \textbf{88.3}                    & \textbf{88.3}                    & \multicolumn{1}{r|}{\textbf{88.3}}   & \textbf{88.3}                    & \textbf{88.3}                    & \textbf{88.3}                    & \textbf{88.3}                    & \textbf{88.3}                    \\
\multicolumn{1}{c|}{}                          & PL & 93.2                             & 93.2                             & 92.8                             & 92.0                             & \multicolumn{1}{r|}{92.3}            & 91.2                             & 93.3                             & \textbf{93.4}                    & \textbf{93.4}                    & \textbf{93.4}                     & 60.0                             & 77.5                             & 84.2                             & 81.4                             & \multicolumn{1}{r|}{87.2}            & 68.6                             & 74.5                             & 77.3                             & 82.5                             & 86.6                             \\ \hline
\multicolumn{1}{c|}{\multirow{2}{*}{Time/s}}   & PS & \textbf{238.7}                   & \textbf{471.6}                   & \textbf{963.0}                   & \textbf{2516.9}                  & \multicolumn{1}{r|}{\textbf{5441.6}} & \textbf{229.9}                   & \textbf{470.7}                   & \textbf{950.1}                   & \textbf{2485.0}                  & \textbf{5436.0}                   & \textbf{265.8}                   & \textbf{525.1}                   & \textbf{1050.3}                  & \textbf{2586.4}                  & \multicolumn{1}{r|}{\textbf{5262.3}} & \textbf{264.7}                   & \textbf{523.1}                   & \textbf{1042.8}                  & \textbf{2627.8}                  & \textbf{5325.1}                  \\
\multicolumn{1}{c|}{}                          & PL & 263.3                            & 497.1                            & 987.1                            & 2562.7                           & \multicolumn{1}{r|}{4991.7}          & 244.7                            & 495.5                            & 992.9                            & 2468.7                           & 5169.8                            & 292.5                            & 583.4                            & 1155.1                           & 3076.9                           & \multicolumn{1}{r|}{5865.6}          & 341.5                            & 592.4                            & 1154                             & 3007.7                           & 5941.1                           \\ \hline
\multicolumn{1}{c|}{\multirow{2}{*}{Verified}} & PS & \multicolumn{1}{c}{\textbf{\checkmark}} & \multicolumn{1}{c}{\textbf{\checkmark}} & \multicolumn{1}{c}{\textbf{\checkmark}} & \multicolumn{1}{c}{\textbf{\checkmark}} & \multicolumn{1}{c|}{\textbf{\checkmark}}    & \multicolumn{1}{c}{\textbf{\checkmark}} & \multicolumn{1}{c}{\textbf{\checkmark}} & \multicolumn{1}{c}{\textbf{\checkmark}} & \multicolumn{1}{c}{\textbf{\checkmark}} & \multicolumn{1}{c|}{\textbf{\checkmark}} & \multicolumn{1}{c}{\textbf{\checkmark}} & \multicolumn{1}{c}{\textbf{\checkmark}} & \multicolumn{1}{c}{\textbf{\checkmark}} & \multicolumn{1}{c}{\textbf{\checkmark}} & \multicolumn{1}{c|}{\textbf{\checkmark}}    & \multicolumn{1}{c}{\textbf{\checkmark}} & \multicolumn{1}{c}{\textbf{\checkmark}} & \multicolumn{1}{c}{\textbf{\checkmark}} & \multicolumn{1}{c}{\textbf{\checkmark}} & \multicolumn{1}{c}{\textbf{\checkmark}} \\
\multicolumn{1}{c|}{}                          & PL & \multicolumn{1}{c}{\textbf{\checkmark}} & \multicolumn{1}{c}{$\times$}           & \multicolumn{1}{c}{$\times$}           & \multicolumn{1}{c}{$\times$}           & \multicolumn{1}{c|}{$\times$}              & \multicolumn{1}{c}{$\times$}           & \multicolumn{1}{c}{$\times$}           & \multicolumn{1}{c}{$\times$}           & \multicolumn{1}{c}{$\times$}           & \multicolumn{1}{c|}{$\times$}           & \multicolumn{1}{c}{\textbf{\checkmark}} & \multicolumn{1}{c}{\textbf{\checkmark}} & \multicolumn{1}{c}{\textbf{\checkmark}} & \multicolumn{1}{c}{\textbf{\checkmark}} & \multicolumn{1}{c|}{$\times$}              & \multicolumn{1}{c}{$\times$}           & \multicolumn{1}{c}{$\times$}           & \multicolumn{1}{c}{$\times$}           & \multicolumn{1}{c}{$\times$}           & \multicolumn{1}{c}{$\times$}           \\ \hline
\end{tabular}

}
\vspace{-1em}
\end{table*}

\begin{table*}
    \caption{The efficacy of single patch moudle}
    \label{tab:rq4_single}
    \centering
    \resizebox{\linewidth}{!}{
    \begin{tabular}{c|rrrrrrrrrr|rrrrrrrrrr}
    \hline
    \multirow{3}{*}{} & \multicolumn{10}{c|}{VGG19}                                                                                                                                                                                                                                         & \multicolumn{10}{c}{ResNet-18}                                                                                                                                                                                                                                     \\ \cline{2-21} 
                      & \multicolumn{5}{c|}{r = 4/255}                                                                                                   & \multicolumn{5}{c|}{r = 8/255}                                                                                                   & \multicolumn{5}{c|}{r = 4/255}                                                                                                   & \multicolumn{5}{c}{r = 8/255}                                                                                                   \\ \cline{2-21} 
                      & \multicolumn{1}{c}{50} & \multicolumn{1}{c}{100} & \multicolumn{1}{c}{200} & \multicolumn{1}{c}{500} & \multicolumn{1}{c|}{1000} & \multicolumn{1}{c}{50} & \multicolumn{1}{c}{100} & \multicolumn{1}{c}{200} & \multicolumn{1}{c}{500} & \multicolumn{1}{c|}{1000} & \multicolumn{1}{c}{50} & \multicolumn{1}{c}{100} & \multicolumn{1}{c}{200} & \multicolumn{1}{c}{500} & \multicolumn{1}{c|}{1000} & \multicolumn{1}{c}{50} & \multicolumn{1}{c}{100} & \multicolumn{1}{c}{200} & \multicolumn{1}{c}{500} & \multicolumn{1}{c}{1000} \\ \hline
    RSR/\%            & 10.0                   & 18.0                    & 16.5                    & 18.2                    & \multicolumn{1}{r|}{16.7} & 10.0                   & 9.0                     & 14.5                    & 11.8                    & 11.6                      & 12.0                   & 18.0                    & 16.5                    & 16.2                    & \multicolumn{1}{r|}{16.2} & 10.0                   & 8.0                     & 12.0                    & 11.6                    & 11.3                     \\ \hline
    RGR/\%            & 10.0                   & 18.0                    & 16.5                    & 18.2                    & \multicolumn{1}{r|}{16.7} & 10.0                   & 9.0                     & 14.5                    & 11.8                    & 11.6                      & 12.0                   & 18.0                    & 16.5                    & 16.2                    & \multicolumn{1}{r|}{16.2} & 10.0                   & 8.0                     & 12.0                    & 11.6                    & 11.3                     \\ \hline
    DSR/\%            & 6.0                    & 4.0                     & 2.5                     & 1.8                     & \multicolumn{1}{r|}{1.3}  & 6.0                    & 6.0                     & 3.0                     & 2.0                     & 1.3                       & 4.0                    & 1.0                     & 3.5                     & 1.8                     & \multicolumn{1}{r|}{1.3}  & 8.0                    & 6.0                     & 4.0                     & 4.8                     & 4.1                      \\ \hline
    DGSR/\%           & 5.8                    & 5.2                     & 3.9                     & 3.9                     & \multicolumn{1}{r|}{3.8}  & 5.7                    & 5.6                     & 2.05                    & 1.9                     & 1.9                       & 5.2                    & 1.5                     & 1.5                     & 1.5                     & \multicolumn{1}{r|}{1.4}  & 7.9                    & 7.2                     & 5.4                     & 4.8                     & 4.8                      \\ \hline
    \end{tabular}
    }
\end{table*}

\subsection{Impact of Patch Module Size and Quantity on Efficiency}


In this experiment, we evaluate the effectiveness of patch module size and quantity in repairing local robustness properties. We perform a comparative analysis of \ours's performance across different sizes and quantities, presenting the experimental results in Table \ref{tab:rq4}. The performance of \ours is assessed using VGG19 and ResNet18 on the CIFAR-10 dataset, across various combinations of patch module scales and quantities, including small patch modules (PS) and large patch modules (PL).

Two distinct patch module sizes are considered: small and large. The small patch module consists of a single linear layer, which takes the output from the network's penultimate layer as input. The large patch module also takes the output from the penultimate layer as input but additionally comprises three hidden layers with 200 neurons each. It is noted that larger patch networks tend to yield poorer results, possibly due to increased over approximation error when passing through ReLU nodes in the hidden layers. As the number of hidden layers increases, the cumulative over approximation error grows, leading to less effective repairs. Smaller patch scales slightly reduce the repair time. Additionally, when using DeepPoly to verify the feature layer repairs after applying smaller patches, the verification success rate is higher.

We also tested the effect of using a single patch to repair the neural network, and the results are shown in 
Table~\ref{tab:rq4_single}. It can be observed that the repair effectiveness using a single patch is significantly 
inferior to using multiple patches. A single patch only achieves 10\% to 20\% RSR and RGR. In contrast, using multiple
patches results in both RSR and RGR reaching 100\%, as shown in Table~\ref{tab:cifar10}. 
Additionally, the single patch performs much worse in terms of DSR and DGSR. 
Therefore, using multiple patches for repair is highly beneficial.

\noindent\doublebox{
    \begin{minipage}{0.96\linewidth}
    {\bf Answer to RQ4:} 
    We observe that smaller patch modules are more effective, likely due to reduced over approximation error and shorter repair times, as well as higher verification success rates with abstract interpretation tools.
    For the aspect of patch module quantity, multiple patch modules significantly outperform only one patch module.

\end{minipage}
}

\section{Related work}
\textit{Provable DNN repair.}
The methods most closely related to ours are ART~\cite{art} and REASSURE~\cite{reassure}.
In comparison to REASSURE, our approach fundamentally differs in its repair objectives.
REASSURE focuses on fixing activation patterns of neural networks,
 yet for high-dimensional data, the input constraints of a property may encompass 
 a large number of activation patterns. In contrast, \ours leverages formal verification to directly repair properties.
Although ART also employs training in its repair process, 
it modifies parameters in the original neural network.
On the other hand, \ours uses patch modules specifically to fix properties, 
ensuring the performance of the original network.
Other methods such as PRDNN \cite{prdnn} and APRNN \cite{aprnn} formulate the repair problem in linear programming. 
However, their efficacy is limited on properties with high-dimensional polytopes, such as the robustness of 
image classification.

\textit{Heuristic DNN repair.}
Utilizing heuristic algorithms such as particle swarm optimization and differential evolution,
CARE~\cite{care} and Arachne~\cite{sohn2022arachne} aim to pinpoint the neurons responsible for faults.
VeRe~\cite{vere_paper} focuses on providing formal verification guidance to assist fault localization,
and defines target intervals for repair synthesis.
For example, DeepRepair~\cite{yu2021deeprepair} and few-shot guided mix~\cite{ren2020few} expand the set of negative samples to generate additional training data.
Conversely, Tian et al.~\cite{tian2018deeptest} augment the existing data by introducing real-world environmental effects, such as fog, to the samples.
DL2~\cite{dl2} fuses logical constraints and loss functions, but without convex-certified guarantees. 

\textit{DNN verification.}
Over the past
decade, various approaches of DNN verification has been proposed including techniques such as
constraint 
solving~\cite{katz2017reluplex,marabou,planet,DBLP:conf/cav/HuangKWW17}, abstract interpretation \cite{gehr2018ai2,singh2019abstract,singh2018fast,li2019analyzing,krelu}, linear relaxation \cite{julian2019verifying,paulsen2020reludiff,xu2020automatic,fastlin,batten2021efficient}, global optimisation \cite{RHK2018,DBLP:conf/nfm/DuttaJST18,DBLP:conf/ijcai/RuanWSHKK19}, CEGAR \cite{atva2020,abstractioncav20,atva2022}, reduction to two-player games \cite{DBLP:conf/tacas/WickerHK18,DBLP:journals/tcs/WuWRHK20}, and star-set abstraction \cite{starset19,imagestar20}. These method offer provable estimations of DNN robustness. Moreover, statistical approaches, presented in \cite{statistical19,DBLP:conf/icml/WengCNSBOD19,baluta2021scalable,quanti5,quanti6,BayesianIJCAI,huangpei01,Probalistic2019,BayesianAAAI,deeppac}, prove to be more efficient and scalable, particularly suited for intricate DNN structures, allowing for the establishment of quantifiable robustness at a specified confidence level. Certified training \cite{colt,TAPS,sabr} uses convex approximation in training the loss function, but it is only used to enhance the robustness of DNN, without making it accessible to repair properties in the training procedure.


\section{Conclusion}

We introduce \ours, a novel approach for property-based repair of local robustness using limited data. Our method provides patch modules as neural networks to repair within the robustness neighborhood, enabling the generalization of this defense to other inputs. In terms of efficiency, scalability, and generalization, our approach surpasses existing methods. 

\section{Acknowledgements}

This work is supported by CAS Project for Young Scientists in Basic Research, Grant No.YSBR-040, ISCAS Basic Research ISCAS-JCZD-202302, and ISCAS New Cultivation Project ISCAS-PYFX-202201.
The limitation of our work is that our patch allocation may assign inappropriate patches when the input is misclassified, and in this case \ours may fail to repair its local robustness.
For future work, we plan to integrate verification methods based on Branch-and-Bound like \cite{reluplex} to enhance the performance of our approach, improving both precision and offering an alternative refinement approach. Additionally, our patch-based repair framework has the potential to evolve into a black-box repair method combined with black-box verification methods like \cite{deeppac}.

\clearpage

\bibliographystyle{IEEEtran}
\bibliography{patchpro}

%
\clearpage
\appendices

\section{Proof of Theorem 3.4}
\setcounter{theorem}{0} 
\begin{theorem}
    Let $\varphi=(F,B(x_i,r))$ be a local robustness property.
    If~$\mathcal L(\varphi) = 0$ on $B(x_i,r)$, i.e.,
    \begin{flalign*}
        \mathcal L^*(\varphi):=&\max (\elmax(\alpha_{\ell}^\mathrm{T},\bm 0) \cdot (x_i+r\cdot\bm 1)+ \elmin(\alpha_{\ell}^\mathrm{T},\bm 0) \cdot  &\\
        & (x_i-r\cdot \bm 1) + \beta_\ell,0) = 0,&        
    \end{flalign*}
    where $\elmax$ and $\elmin$ are the element-wise $\max$ and $\min$ operation, $\bm 0$ and $\bm 1$ are the vector in $\mathbb R^{n_0}$ with all the entries $0$ and $1$, respectively, then the property $\varphi$ holds. 
\end{theorem}
\begin{proof}
    As $\alpha_\ell^\mathrm{T} x + \beta_\ell$ is a linear function with respect to $x$, we can calculate the maximum of $\alpha_\ell^\mathrm{T} x + \beta_\ell$ on ${B(x_i,r)}$ as follows: 
    \begin{align*}
          &\quad \quad\max_{x\in B(x_i,r)}\alpha_\ell^\mathrm{T} x + \beta_\ell \\
        &= \sum_{i \in \mathbb R^{n_0}}\mathbbm{1}_{\{(\alpha_\ell)_{i} > 0\}}(\alpha_\ell)_{i} \max_{x\in B(x_i,r)}x \\
        & + 
        \sum_{i \in \mathbb R^{n_0}}\mathbbm{1}_{\{(\alpha_\ell)_{i} < 0\}}(\alpha_\ell)_{i} \max_{x\in B(x_i,r)}x + \beta_\ell \\
        &= \sum_{i \in \mathbb R^{n_0}}\mathbbm{1}_{\{(\alpha_\ell)_{i} > 0\}}(\alpha_\ell)_{i} \cdot (x_i+r) \\
        &+ \sum_{i \in \mathbb R^{n_0}}\mathbbm{1}_{\{(\alpha_\ell)_{i} < 0\}}(\alpha_\ell)_{i} \cdot (x_i-r) + \beta_\ell 
    \end{align*}
    Then we have
    \begin{flalign*}
        & \mathcal L(\varphi)(x) = \sum_{\ell \ne \ell_0} \max (\alpha_\ell^\mathrm{T} x + \beta_\ell,0) = 0,\forall x \in B(x_i,r)&\\
        \iff &\sum_{\ell \ne \ell_0}\max(\elmax(\alpha_{\ell}^\mathrm{T},\bm 0) \cdot (x+r\cdot\bm 1)+\elmin(\alpha_{\ell}^\mathrm{T},\bm 0)  &\\
        &\cdot (x-r\cdot \bm 1) + \beta_\ell,0) = 0&\\
        \iff &\mathcal L^*(\varphi) = 0.&\\ 
    \end{flalign*}
    Therefore, if $\mathcal L^*(\varphi) = 0$, we have
    \begin{align*} 
    \forall x \in B(x_i,r), F(x)_\ell-F(x)_{\ell_0} \le \alpha_\ell^\mathrm{T} x + \beta_\ell \leq 0,
    \end{align*}
    then the property $\varphi$ holds. 
\end{proof}

\section{An explainable example}

\begin{figure}[htbp]
    \centering
    \scalebox{0.8}{
   \begin{tikzpicture}[->,>=stealth,auto,node distance=1.2cm,semithick,scale=1,every node/.style={scale=1}]
      \tikzstyle{blackdot}=[circle,fill=black,minimum size=6pt,inner sep=0pt]
      \tikzstyle{state}=[minimum size=0pt,circle,draw,thick]
      \tikzstyle{stateNframe}=[minimum size=0pt]	
      \node[state](x1){$x_1$};
      \node[state](x2)[below of=x1,yshift=-0.3cm]{$x_2$};
      \node[state](x3)[right of=x1,xshift=1.2cm]{$x_3$};
      \node[state](x4)[right of=x2,xshift=1.2cm]{$x_4$};
      \node[state](x5)[right of=x3,xshift=1.2cm]{$x_5$};
      \node[state](x6)[right of=x4,xshift=1.2cm]{$x_6$};
      \node[state](y1)[right of=x5,xshift=1.2cm]{$y_1$};
      \node[state](y2)[right of=x6,xshift=1.2cm]{$y_2$};
  
    \path (x1) edge	[-]						    node[font=\footnotesize] {$0.8$} (x3)
          (x1) edge[-]							node[font=\footnotesize][xshift=-0.3cm,yshift=0.25cm] {$1.4$} (x4)
          (x2) edge	[-]						    node[font=\footnotesize][xshift=0.4cm,yshift=-0.7cm] {$1.1$} (x3)
          (x2) edge[-]							node[font=\footnotesize][below] {$1.2$} (x4)
          (x3) edge	[-]						    node[font=\footnotesize] { $\mathrm{ReLU}$} (x5)
          (x4) edge[-]							node[font=\footnotesize] {$\mathrm{ReLU}$} (x6)
          (x5) edge	[-]						    node[font=\footnotesize] {$-0.8$} (y1)
          (x5) edge[-]							node[font=\footnotesize][xshift=-0.3cm,yshift=0.25cm] {$1.1$} (y2)
          (x6) edge	[-]						    node[font=\footnotesize][xshift=0.4cm,yshift=-0.7cm] {$0.4$} (y1)
          (x6) edge[-]							node[font=\footnotesize][below] {$-1.1$} (y2);
  \end{tikzpicture}}
    \caption{A fully connected neural network $N$ with ReLU activations.} \label{fig:example1}
    \vspace{-1em}
  \end{figure}
  
\begin{example}
    Consider the neural network in Fig.~\ref{fig:example1} and
    the inputs $x = (-0.7,1)$ labeled ``2''.
    Within the region $B(x,0.5)$, we have its counterexample $x^* = (-0.2, 1.5)$, which violates the local robustness. 
    To repair the network, we need to construct a patch $P=wx$ where $w\in\mathbb{R}^{2\times 2}$.

    After initializing $w$ to all $0.1$,    
we train the patch according the ``while'' loop begins at Line~7 in Alg.~\ref{alg:advrepair}. As shown in Line~5 in Alg.~\ref{alg:training}), we first execute DeepPoly to obtain the safety violated loss function $\mathcal{L}=w_{1,1}x_1+w_{1,2}x_2 - w_{2,1}x_1-w_{2,2}x_2 + (0.7x_1 + 0.14x_2 + 1.08)$. To maximize $\mathcal{L}$, we set $x_1=-0.2$ and $x_2=1.5$, then we have $\mathcal{L}^*(w)=-0.2w_{1,1} + 1.5w_{1,2}+0.2w_{2,1}-1.5w_{2,2}+1.15$. Finally, we update $w$ by $w-\eta\nabla \mathcal L^*(w)$ with the learning rate $\eta$. This process is repeated until the robustness is proven or the epochs reaches its limit $R$. 

To show the subsequent repair process, we set $\eta=0.6$ and $R=1$. Under this setting, the network is not be repaired, then we need to refine the robustness to two properties by bisecting the region $B(x,0.5)$ (see Line ~14--19 in Alg.~\ref{alg:advrepair}). 
Specifically, by the judgment that $\partial_1\mathcal{L} > \partial_2\mathcal{L}$, we select the dimension of $x_1$ and divide the robustness region $B(x,0.5)$ into $[-1.2,-0.7] \times [0.5,1.5]$ and $[-0.7,-0.2] \times [0.5,1.5]$. We perform the above repair process again, and based on a more accurate abstraction provided by two new properties, we finally obtain the repaired network with the patch 
\[
P=  \begin{pmatrix}
        0.22 & -0.8 \\
        -0.02 & 1
    \end{pmatrix}\begin{pmatrix}
            x_1 \\
            x_2
        \end{pmatrix}\,.
\]
\end{example}

\section{Limitation}
As mentioned in section~\ref{allocation}, the adversarial examples not within any known robustness regions may lead to the establishment of incorrect new properties, and the resulting inappropriate patches allocation may ultimately affect the performance of the repaired neural network.
Although this situation is rare in reality, we created an extreme setting here to investigate this weakness of our method.
Specifically, considering the local robustness with radius $r=4/255$, we reuse the repaired VGG19 for CIFAR-10 and execute it over the adversarial dataset $D_\mathrm{adv}$ consisting of the adversarial examples generated by attacking the original network over the testset $D_\mathrm{t}$. The attack utilize AutoAttack with the step size of 10.
By preserving the new properties established in this execution,
we retest the drawdown of the repaired model over $D_\mathrm{t}$.
The results are presented in Fig.~\ref{fig:adv_dd}, showing that the accuracy of the repaired model drops significantly under this extreme setting.
This is a limitation of \ours, and it also echoes the statement aforementioned that adversarial detection can serve as an important supplement to our method.

\begin{figure}[t]
    \centering
    \includegraphics[width=0.7\linewidth]{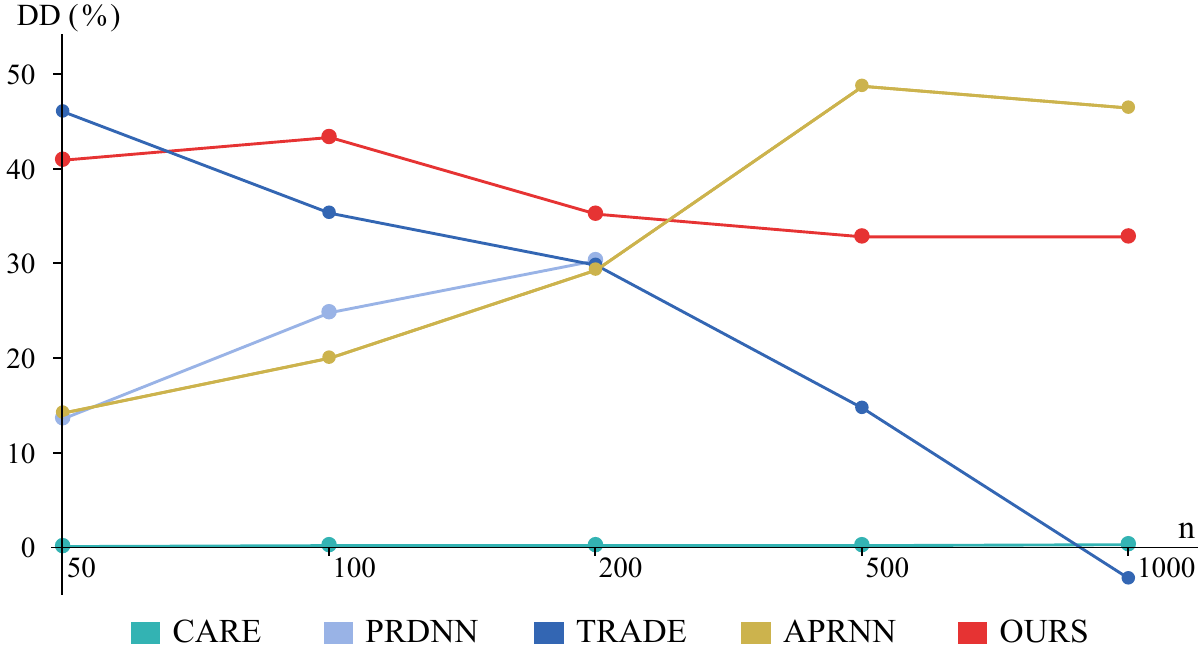}
    \caption{Results under the extreme setting.
    }
    \label{fig:adv_dd}
    \vspace{-1em}
\end{figure}

\section{Network architectures and accuracies of the DNN trained on MNIST}
The network architectures and accuracies of the DNN trained on MNIST are detailed in Table~\ref{tab:mnist-arc}. 
\label{appendix:exp}

\begin{table}[t]
\caption{Network architectures and accuracies of DNN trained on MNIST dataset}
\label{tab:mnist-arc}
\centering
\begin{tabular}{c|c|c}
\hline
Name                        & Accuracy/\%           & Model structure                       \\ \hline
\multirow{6}{*}{FNN\_small} & \multirow{6}{*}{96.6} & linear layer of 50 hidden units       \\
                            &                       & linear layer of 50 hidden units       \\
                            &                       & linear layer of 50 hidden units       \\
                            &                       & linear layer of 50 hidden units       \\
                            &                       & linear layer of 32 hidden units       \\
                            &                       & linear layer of 10 hidden units       \\ \hline
\multirow{6}{*}{FNN\_big}   & \multirow{6}{*}{97.2} & linear layer of 200 hidden units      \\
                            &                       & linear layer of 200 hidden units      \\
                            &                       & linear layer of 200 hidden units      \\
                            &                       & linear layer of 200 hidden units      \\
                            &                       & linear layer of 32 hidden units       \\
                            &                       & linear layer of 10 hidden units       \\ \hline
\multirow{3}{*}{CNN}        & \multirow{3}{*}{98.3} & Conv2d(1, 16, 4, stride=2, padding=1) \\
                            &                       & linear layer of 100 hidden units      \\
                            &                       & linear layer of 10 hidden units       \\ \hline
\end{tabular}

\end{table}

\end{document}